\documentclass[lettersize,journal, final]{IEEEtran}

\usepackage{amsmath,amssymb,amsfonts, color,nicefrac}
\usepackage{bm}
\usepackage{enumitem}

\newcommand{\dd}{\mathrm{d}}
\newcommand{\N}{\mathcal{N}}
\DeclareMathOperator*{\argmin}{arg\,min}
\newtheorem{remark}{\textbf{Remark}}
\newtheorem{proposition}{\textbf{Proposition}}
\newtheorem{theorem}{\textbf{Theorem}}
\newenvironment{proof}{\par\noindent\textbf{Proof.}\ }{\hfill$\blacksquare$\par}

\usepackage{caption}
\usepackage{subcaption}

\usepackage{algorithm}
\usepackage{algpseudocode}
\usepackage{array}
\usepackage{textcomp}
\usepackage{stfloats}
\usepackage{url}
\usepackage{verbatim}
\usepackage{graphicx}
\def\BibTeX{{\rm B\kern-.05em{\sc i\kern-.025em b}\kern-.08em
    T\kern-.1667em\lower.7ex\hbox{E}\kern-.125emX}}
\usepackage{balance}
\begin{document}
\title{Communication-Aware Map Compression for Online
Path-Planning: A Rate-Distortion Approach}
\author{Ali Reza Pedram, Evangelos Psomiadis, \IEEEmembership{Student Member,~IEEE}, Dipankar Maity, \IEEEmembership{Senior Member,~IEEE}, Panagiotis Tsiotras,~\IEEEmembership{Fellow,~IEEE} 
\thanks{This work was supported by the ARL grant DCIST CRA W911NF-17-2-0181. A.R. Pedram, E.  Psomiadis, and P. Tsiotras are with the Department of Aerospace Engineering of Georgia Institute of Technology. D. Maity is with the Department of Electrical and Computer Engineering,
University of North Carolina at Charlotte.
Corresponding author: \texttt{apedram@gatech.edu}}}

% \markboth{Journal of \LaTeX\ Class Files,~Vol.~18, No.~9, September~2020}%
% {How to Use the IEEEtran \LaTeX \ Templates}

\maketitle

\begin{abstract}
This paper addresses the problem of collaborative navigation in an unknown environment, where two robots, referred to in the sequel as the \textit{Seeker} and the \textit{Supporter}, traverse the space simultaneously. The Supporter assists the Seeker by transmitting a compressed representation of its local map under bandwidth constraints to support the Seeker's path-planning task. We introduce a bit-rate metric based on the expected binary codeword length to quantify communication cost. Using this metric, we formulate the compression design problem as a rate-distortion optimization problem that determines \textit{when} to communicate, \textit{which} regions of the map should be included in the compressed representation, and at \textit{what resolution} (i.e., quantization level) they should be encoded. Our formulation allows different map regions to be encoded at varying quantization levels based on their relevance to the Seeker's path-planning task. We demonstrate that the resulting optimization problem is convex, and admits a closed-form solution known in the information theory literature as reverse water-filling, enabling efficient, low-computation, and real-time implementation. Additionally, we show that the Seeker can infer the compression decisions of the Supporter independently, requiring only the encoded map content and not the encoding policy itself to be transmitted, thereby reducing communication overhead. Simulation results indicate that our method effectively constructs compressed, task-relevant map representations, both in content and resolution, that guide the Seeker’s planning decisions even under tight bandwidth limitations.
\end{abstract}


\section{Introduction}
\IEEEPARstart{A}{utonomous} navigation in unknown environments is essential for many real-world robotic applications, including search and rescue missions \cite{queralta2020collaborative}, agricultural surveys, and planetary exploration \cite{rockenbauer2024traversing}. For a ground robot to traverse such environments efficiently, it must incrementally build an understanding of its surroundings and plan its paths accordingly. However, directly exploring unknown terrain can be costly, inefficient, and risky --- potentially leading to entrapment, electro-mechanical damage, or mission failure. To mitigate these risks, a common strategy is to employ a second robot --- a lightweight ``scouting agent"  such as an aerial drone --- that explores the environment in parallel. The scouting robot can collect environmental data from a broader vantage point and relay this information to the ground robot, which we refer to as the \textit{Seeker}. This separation of roles allows the ground robot to focus on payload transport or precise operations, while the scout, hereafter referred to as the \textit{Supporter}, aids in a more efficient and informative route planning. 

\begin{figure}[tb]
    \centering        
{\includegraphics[width=0.7\linewidth]{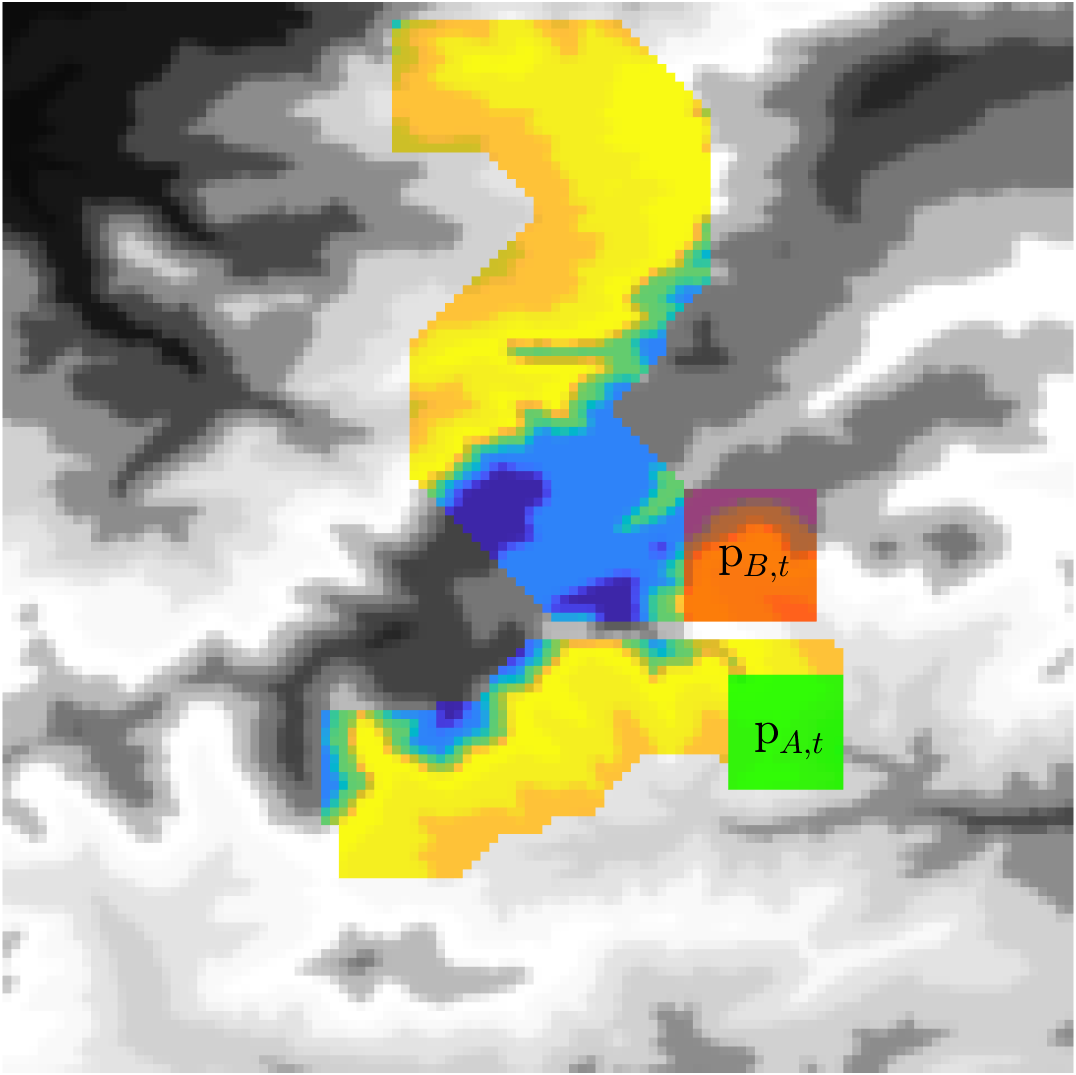}
    \caption{A sample traversability map. The positions of the Seeker and the Supporter are denoted by $\textup{p}_{A,t}$ and $\textup{p}_{B,t}$, respectively. The currently observable cells are highlighted with green and orange rectangles for the Seeker and the Supporter, respectively. Unobserved cells are shown in gray, with shading intensity proportional to traversability values.
    Previously observed parts of the map are color-coded based on their traversability values; lighter areas represent regions that are easier to traverse.} 
    \label{fig:map}} 
\end{figure}

We assume a collaborative navigation setting between the Seeker and the Supporter when both robots are equipped with onboard sensors and collect local observations as they move, as demonstrated in Fig.~\ref{fig:map}. The Seeker aims to reach a known destination while minimizing traversal cost, guided by an online path-planning algorithm. 
The Supporter observes the environment and periodically sends its observations to the Seeker to aid its navigation task, under the assumption that the environment's overall size and the goal location are known to both agents. 
We assume that the Supporter follows a predetermined path or one generated using standard path-planners such as \cite{sasaki2020map, rockenbauer2024traversing, psomiadis2024multi}, 
designed to maximize environmental coverage or information gain. To support coordination, the Supporter also receives the Seeker’s predictive path at each time step.

Sending the Supporter's entire local map to the Seeker requires substantial communication bandwidth, which may not be available in real-world scenarios characterized by limited or unreliable connectivity.
Moreover, there is the possibility that the observed map is marginally relevant to the Seeker's planning task, and thus communicating the map would drain the robots' resources (e.g., communication bandwidth or computation power) or increase latency (i.e., the computation time required to analyze a newly received local map) with little to no benefit. 
Thus, it is necessary to think about how we can strategically manage the communication resources.

In this paper, we present what is, to the best of our knowledge, the first framework that enables the design of informative and compressed representations of the Supporter's local map for use in communication-aware path-planning. These compressed representations are shared with the Seeker and are used to guide its decisions while navigating under bandwidth constraints. Our key contributions are as follows:

\begin{itemize}
    \item We introduce a communication-aware map compression policy that determines \textit{when} communication should occur, \textit{which} parts of the observed environment are relevant, and at \textit{what resolution (quantization level)} they should be transmitted from the Supporter to the Seeker, to support efficient path-planning in bandwidth-limited settings.
    \item The proposed compression scheme is based on the (Gaussian) rate-distortion formulation, which selects and compresses map content based on its contribution to the Seeker's path-planning task. We show that the resulting optimization problem admits a closed-form solution, enabling efficient on-the-fly implementation during navigation.
    \item We show that the Seeker can predict the compression decision at each time step by replicating the Supporter's compression logic. As a result, the Supporter only needs to transmit the compressed content, eliminating the need to communicate the compression policy itself, thereby reducing overhead and preserving bandwidth.
    \item We validate our approach in simulation, demonstrating that it achieves effective navigation performance under strict communication constraints.
\end{itemize}

\textbf{Notation:} Matrices and vectors are represented by uppercase and lowercase letters, respectively. We denote the set of integers by $\mathbb{Z}$, and the $d$-dimensional Euclidean space by $\mathbb{R}^d$. 
For a vector $x \in \mathbb{R}^d$, the $i$-th element of $x$ is denoted by $[x]_i$, where $i = 1, \dots, d$.
The set of symmetric positive semi-definite $d \times d$ matrices is defined as
$\mathbb{S}_d^{+} \triangleq \{ P \in \mathbb{R}_d \times \mathbb{R}_d : P=P^\top, P \succeq 0\}$.
% where the notation $P \succeq 0$ implies $v^\top P v \geq 0$ for all $v \in \mathbb{R}^d $.  
For two symmetric matrices $P$ and $Q$, $P \succeq Q$ implies $P-Q \succeq 0$. The functions $\mathrm{tr}(\cdot)$, $\mathrm{rank}(\cdot)$, and $\det(\cdot)$ denote the trace, rank, and determinant of a matrix, respectively.
We use $\mathrm{diag}_{1 \leq i \leq d}(x_i) $ to denote the $d \times d$ diagonal matrix with diagonal entries $x_1, \dots, x_d$. The set of all $d \times d$ diagonal matrices is denoted by $ \mathbb{D}_d$. We use $x_{0:t}$ to denote the sequence $\{x_0, \dots, x_t\}$, and $I_d$ and $\textbf{1}_d$ to denote the $d \times d$ identity matrix and $d$-dimensional column vector of all ones, respectively. Variables associated with the Seeker are indexed by subscript 
$A$, while those associated with the Supporter are indexed by subscript $B$.

Random variables are denoted by bold symbols (e.g., $\bm{x}$), and
their realization by non-bold symbols (e.g., $x$). 
The expressions $\mathbb{E}[\bm{x}]$ and $\mathbb{V}[\bm{x}]$ are the mean and variance of $\bm{x}$, respectively. The \textit{differential entropy} of $\bm{x}$ having probability density function  $f(x)$ is defined as $h(\bm{x}) \triangleq -\int f(x) \log(f(x)) \dd x$.
The \textit{Kullback–Leibler $(\mathrm{KL})$ divergence} between two random variables $\bm{x}$ and $\bm{y}$ with respective probability distribution functions $f(x)$ and $g(x)$ is defined as 
$\mathrm{KL}(\bm{x}\|\bm{y}) \triangleq \int f(x) \log(\frac{f(x)}{g(x)}) \dd x$.
Also, $h(\bm{x}|\bm{y})$ denotes the \textit{conditional entropy} of $\bm{x}$ given $\bm{y}$. 
The \textit{mutual information} between $\bm{x}$ and $\bm{y}$ defined as $I(\bm{x}; \bm{y})\triangleq h(\bm{x})-h(\bm{x}|\bm{y})= h(\bm{y})-h(\bm{y}|\bm{x})$. The $d$-dimensional Gaussian random variable with mean $\mu\in \mathbb{R}^d$ and covariance $P \in \mathbb{S}_{+}^{d}$ is denoted by $\bm{x}\sim \mathcal{N}(\mu, P)$. The entropy of $\mathcal{N}(x, P)$ is  $h(\bm{x})= \frac{d}{2}\log(2\pi e) + \frac{1}{2}\log\!\det(P)$.  The scalar uniform distribution between $a$ and $b$ is denoted by $\bm{y} \sim \mathcal{U}(a, b)$, where $h(\bm{y})= \log(b-a)$.

\section{Related Work}
Multi-agent robotic systems have been extensively studied for autonomous navigation and exploration due to their scalability, adaptability, and ability to coordinate effectively. Heterogeneous teams, in particular, are valued for combining diverse sensing and mobility capabilities, as in the Seeker-Supporter setting considered in this paper. Prior surveys have reviewed coordination strategies, communication architectures, and task allocation mechanisms that support such systems \cite{grocholsky2006cooperative, cladera2024challenges}.

One major research direction focuses on informative planning, where the Supporter’s trajectory is optimized to improve the Seeker’s decision-making. For example, \cite{sasaki2020map} proposes an iterative framework that uses hyper-belief propagation to select observations that reduce localization uncertainty. Similarly, \cite{rockenbauer2024traversing} presents a cost-aware informative path-planning strategy that aims to minimize the Supporter’s traversal cost, offering formal guarantees and outperforming frontier-based heuristic methods such as \cite{isler2016information}.

Another emerging direction is semantically informed collaboration, where agents exchange high-level representations of the environment instead of raw sensor data. Semantic mapping frameworks, such as \cite{yue2020collaborative}, enable robots to share task-relevant scene understanding in a compact form. Similarly, distributed SLAM methods like DDF-SAM \cite{cunningham2013ddf} support consistent map fusion across agents without requiring full data exchange. 

A third line of research focuses on communication-aware map compression, where the goal is to transmit compact, task-relevant data while minimizing bandwidth usage. D-Lite \cite{chang2023d}, for example, compresses 3D scene graphs to prioritize navigation-relevant content under communication constraints. Similarly, \cite{Damigos2024} proposes a control function for 5G-enabled multi-robot systems that adapts the transmission rate based on network KPIs, enabling real-time 3D point cloud map merging via edge servers while avoiding data congestion. Compression selection methods, such as \cite{psomiadis2025communication}, have been proposed to select the appropriate compression scheme from a predefined set of compression templates. 
Although these approaches effectively reduce data transmission and improve the Seeker’s path-planning performance, they primarily focus on selecting among existing compression schemes rather than designing new ones adapted to the task.

In contrast, the approach proposed in this paper addresses the design of task-aware compression schemes based on a rate-distortion trade-off. This aligns with the broader framework of \textit{minimum information control} \cite{tishby2000information}, where the goal is to minimize the mutual information between sensing and action while preserving performance. Related studies in control theory address the trade-off between task effectiveness and communication cost through quantizer selection \cite{maity2021optimal, maity2023optimal}, and belief-space planning approaches such as \cite{pedram2022gaussian, hibbard2023simultaneous} optimizing for both motion and information cost. This growing body of work highlights a shift toward autonomy that is aware of both mission goals and the informational cost of achieving them.

%%%%%%%%%%%%%%%%%%%%%%%%%%%%%%%%%%%%%%%%%%%%%%%%%%%% 
\begin{figure*}
    \centering
{\includegraphics[trim = 2.6cm 1.8cm 0.4cm 0.9cm, clip, width=0.97\linewidth]{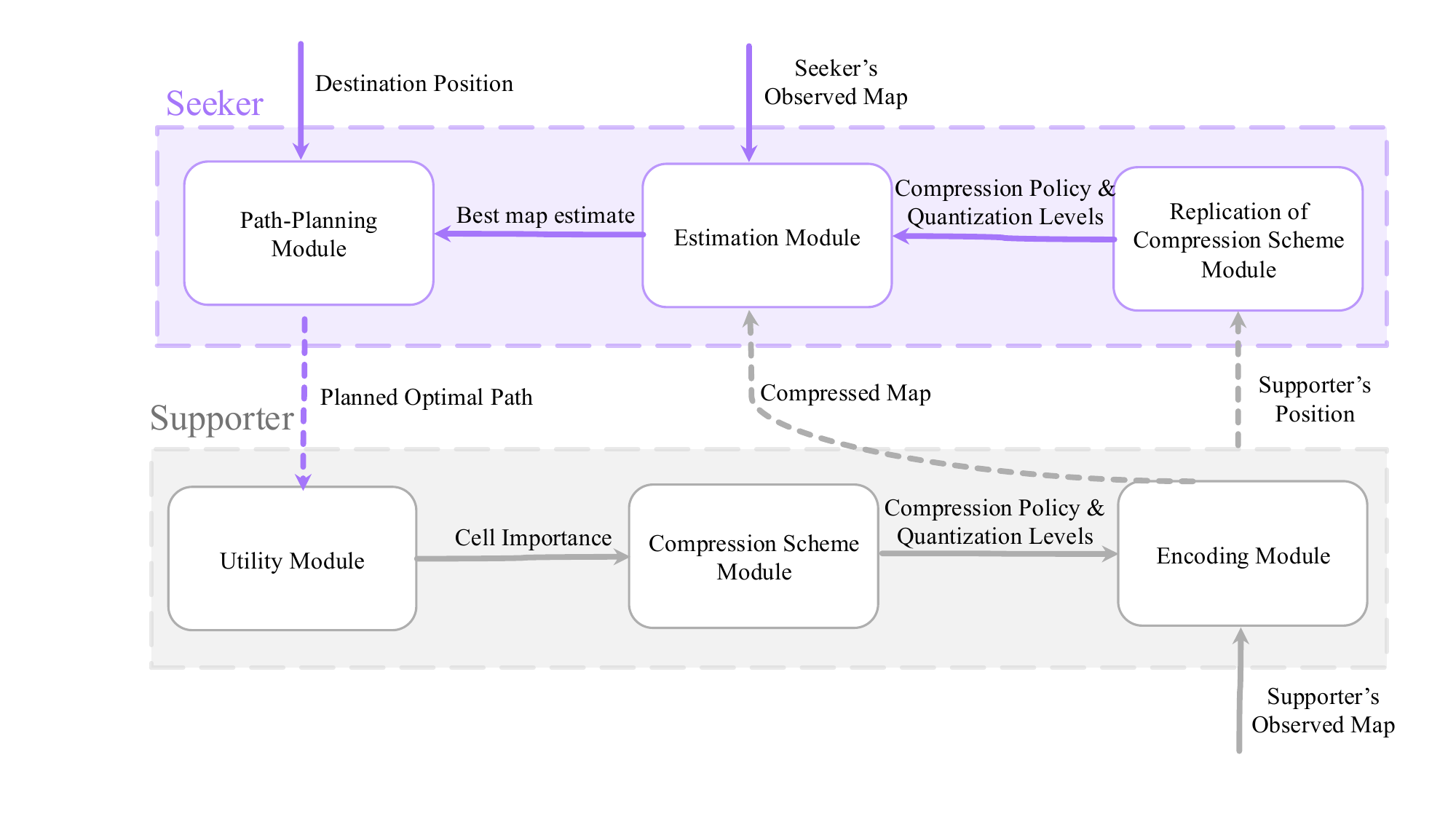}}
    \caption{The schematic of the proposed framework. At each time step, the Seeker sends its current optimal path to the Supporter. the Supporter uses this information to select the optimal compression of its local map and sends the compressed map and its own location to the Supporter.}
    \label{fig:framework}
\end{figure*}
%%%%%%%%%%%%%%%%%%%%%%%%%%%%%%%%%%%%%%%%%%%%%%%%%%%% 
% \begin{figure*}
%     \centering
% {\includegraphics[trim = 0.6cm 0.9cm 0.1cm 0.6cm, clip, width=0.77\linewidth]{figs/Framework_3.pdf}}
%     \caption{The schematic of the proposed framework. At each time step, the Seeker sends its current optimal path to the Supporter. the Supporter uses this information to select the optimal compression of its local map and sends the compressed map and its own location to the Supporter.}
%     \label{fig:framework}
% \end{figure*}
%%%%%%%%%%%%%%%%%%%%%%%%%%%%%%%%%%%%%%%%%%%%%%%%%%%% 

\section{System Overview}
This section provides a high-level overview of the collaborative navigation framework, serving as necessary background to introduce the main question of this work, namely the design of optimal compression mechanisms, addressed in the subsequent sections. We first describe the environment and its map representation, whose overall size is known to both agents, along with the key system components.

We assume that the environment is represented by a traversability grid map $M$ in 2D with $d$ cells. Fig.~\ref{fig:map} represents an example of $M$ with $d = 128 \times 128$.
The traversability values of the grid cells are stored in a vector $x \in \mathbb{R}^{d}$, where $[x]_j$ is in the range $[0,1]$ for all $j = 1, \dots, d$. 
The value $[x]_j = 1$ indicates a fully-occupied (and thus untraversable) cell, whereas $[x]_j = 0$ denotes a free cell.  In Fig.~\ref{fig:map}, the darker regions indicate areas more difficult to traverse, such as obstacles and rough terrain.
In contrast, the lighter regions with a yellow hue indicate areas that are easier to traverse, like clear paths and smoother terrain.

Fig.~\ref{fig:framework} illustrates the architecture of the proposed framework.
The Supporter follows a predefined path, possibly one designed to maximize information gathering. At each time step, the Supporter receives the Seeker’s current planned path, which includes its current position, destination position, and intermediate positions along the path. The Supporter collects its own local observations and uses the Utility Module to evaluate the relevance (or importance) of its observed map cells based on the Seeker's current path. Next, the Supporter solves a rate-distortion optimization problem to compute an informative and bandwidth-efficient compression of its local map. 
Finally, the selected compression policy is applied to the local observation, producing a compressed representation, which is subsequently quantized based on the designed quantization levels. Finally, the compressed and quantized data is entropy-coded and transmitted to the Seeker along with the Supporter's current location.

On the Seeker side, the robot first uses the Supporter's transmitted position and a replica of the Supporter's compression scheme to infer the compression map and quantization levels used to generate the received compressed signal. 
Knowing this, the Seeker then combines the received compressed map data with its own sensor observations and updates its belief over the traversability map. The updated belief is then used by the Seeker's Path-Planning Module to compute an optimal path to the destination at the next time step.

In the remainder of this section, we describe the observation models and the standard processing modules common to multi-robot navigation frameworks. The Utility Module and Compression Scheme Module, introduced specifically to enable communication-aware compression, are presented separately in Section~\ref{sec:compression}.

\subsection{Observation Models}
\label{subsec:obs}
As mentioned earlier, the Seeker is unaware of the traversability values of the map $x$. 
At time $t$, the Seeker's position is denoted by $\textup{p}_{A,t}$. From this position, it can only observe the cells within its field of view (FOV), and this observation is subject to noise. For instance, the green rectangle in Fig.~\ref{fig:map} shows the cells in the FOV at $\textup{p}_{A,t}$.
By arbitrarily ordering the cells within the Seeker's FOV, we represent the Seeker's local map at time $t$ as the vector $x_{A,t}\in \mathbb{R}^{d_A}$, where $d_A$ is the size of the local map. 
The relationship between the observed area 
$x_A$ and the complete map $x$ is characterized by a linear relation $x_{A,t}=C_{A,t} x$. 
The matrix $C_{A,t} \in \mathbb{R}^{d_A \times d}$ is a $\{0, 1\}$ matrix with $[C_{A,t}]_{ij}=1$ if $[x]_j$ is the $i$-th element in $x_{A,t}$, and $[C_{A,t}]_{ij}=0$, otherwise. It is worth noting that $C_{A,t}$ is solely a function of $\textup{p}_{A,t}$.  We assume that the Seeker observes $x_{A,t}$ with additive Gaussian noise:
\[\bm{y}_{A,t} = x_{A,t} + \bm{m}_t =C_{A,t} x +\bm{m}_t,\]
where $\bm{m}_t \sim \N(0, \sigma_m^2 I_{d_A})$ represents independent, zero-mean Gaussian noise across observed cell, with variance $\sigma_m^2$. 

The Supporter is also equipped with onboard sensing capability that can observe its local map but without noise. For example, the orange rectangle in Fig.~\ref{fig:map} shows the cells in the FOV at $\textup{p}_{B,t}$.
Similar to the Seeker, we denote the Supporter's position and local map at time $t$ as $\textup{p}_{B,t}$ and $x_{B,t}\in \mathbb{R}^{d_B}$, respectively, where $d_B$ denotes the size of the Supporter's FOV. 
Likewise, the Supporter's local map is related to the global map by the linear relation map $x_{B,t} = C_{B,t} x$, where $C_{B,t} \in \mathbb{R}^{d_B\times d}$ is also a $\{0,1\}$ matrix, and is solely a function of $\textup{p}_{B,t}$.
We denote the cells outside of the Supporter's local map by $x_{O,t} \in \mathbb{R}^{d-d_B}$, which is also linearly related to the global map as $x_{O,t}=C_{O,t} x$ with a $\{0, 1\}$ matrix $C_{O,t} \in \mathbb{R}^{(d-d_B)\times d}$. 
If we define the re-ordered map state $x_{R,t}=[x_{B,t}^\top, x_{O,t}^\top ]^\top\!$, $x_{R,t}$ is mapped linearly to $x$ by $x_{R,t}=C_{R,t} x$, with  $C_{R,t}=[C_{B,t}^\top, C_{O,t}^\top ]^\top \in \mathbb{R}^{d\times d}$. 
One may verify that the matrix $C_{R,t}$ is a permutation matrix i.e., a $\{0,1\}$ unitary matrix ($C_{R,t} C_{R,t}^\top=C_{R,t}^\top C_{R,t} = I$). 

\begin{remark}
    The noise-free assumption for the Supporter’s observations is made to simplify the presentation. This assumption does not limit the generality of our approach and is formally relaxed in Section~\ref{subsec:noisy}.
\end{remark}

\subsection{Encoding Module and Communication Scheme}
\label{subsec:com}

Due to communication constraints, it is often impossible (or maybe not strategic) to send the Supporter's local map $x_{B,t}$ to the Seeker. 
Instead, it is helpful to compress and only transmit a compressed version of $x_{B,t}$. 
In this paper, we restrict ourselves to a linear compression scheme,  described by a linear mapping $\Theta: \mathbb{R}^{d_B} \rightarrow \mathbb{R}^{d_\Theta}$ 
 between the full-resolution map space and the compressed map space, with $0 \leq d_{\Theta} \leq d_B$. 
That is, for a given full-resolution local map of the Supporter $x_{B,t}$ and compression matrix $\Theta_t \in \mathbb{R}^{d_{\Theta_t} \times d_B}$, the compressed representation is computed as 
\[
o_{B,t} = \Theta_t x_{B,t}, \quad o_{B,t} \in \mathbb{R}^{d_{\Theta_t}}.
\]
This linear compression scheme can be interpreted as computing features by aggregating information across spatial regions, where each feature is formed as a weighted sum over the original map cells. 
In this view, $d_{\Theta_t}$ represents the number of such features, with each row of 
$\Theta_t$ defining a specific aggregation pattern. We assume $\Theta_t$ is full rank, meaning $\mathrm{rank}(\Theta_t)=d_{\Theta_t}$. 
Thus, $d_{\Theta}=0$ corresponds to sending a null signal (i.e., no communication), while $d_{\Theta}=d_B$ corresponds to sending the full-resolution map.
This formulation allows the Supporter to flexibly share partial information or suppress communication entirely when the map content is not useful for the Seeker's planning task.

To transmit the continuous-value vector $ o_{B,t}$ over a digital channel, it must be quantized, as digital communication systems can only transmit finite-bit representations. 
While quantization is necessary for communication over digital channels, in our framework it also serves a second, more strategic role; namely, as an additional degree of freedom to manage communication cost. 
Specifically, the Supporter can deliberately adjust the quantization resolution to reduce communication overhead. 

Direct quantization of $o_{B,t}$ can cause problems for inference on the Seeker’s side.
The Seeker does not have access to the exact value of $o_{B,t}$; instead, it maintains a belief about the compressed map.
Upon receiving a quantized signal, the Seeker must infer the associated quantization error, which itself becomes a random variable whose distribution depends on the signal and is generally non-uniform. 
This signal-dependent and non-uniform nature of the error complicates modeling and estimation. 
Since such communications repeat over time, the inference becomes increasingly intractable.

To mitigate this challenge, we adopt \textit{dithered quantization}, a standard technique that makes the quantization error have a uniform distribution and renders it independent of the signal, making it analytically tractable.
Specifically, we adopt the \textit{entropy-coded dithered quantization (ECDQ)} scheme~\cite{derpich2012improved}, where each component $[o_{B,t}]_i$ 
is perturbed by a random dither $\eta_{i,t}$ sampled uniformly from $\bm{\eta}_{i,t} \sim \mathcal{U}\left(-\Delta_{i,t}/2,\Delta_{i,t}/2\right)$, producing
\begin{align*}
    [z_{t}]_{i} = [o_{B,t}]_i + \eta_{i,t}.
\end{align*}
The perturbed signal $[z_{t}]_{i}$ is then passed through a scalar uniform quantizer with sensitivity $ \Delta_{i,t}$, resulting in
\begin{align*}
    &q_{i,t}=Q_{\Delta_{i,t}}([z_{t}]_i) = k \Delta_{i,t}, 
\end{align*}
where $k \in \mathbb{Z}$ is an integer such that
\begin{align*}
   (k-1/2) \Delta_{i,t} \!\leq\! [z_{t}]_i\!<\! (k+1/2) \Delta_{i,t}. 
\end{align*}
Essentially, the function $Q_{\Delta_{i,t}}(\cdot)$ rounds its argument to the nearest multiple of $\Delta_{i,t}$.
Next, $q_{i,t}$ is encoded into a binary code-word by an entropy coding scheme \cite[Chapter~5]{cover1999elements} and is then transmitted to the Seeker over a binary channel.

The Seeker receives $q_{i,t}$ and interprets it as the realization of the random variable 
$\bm{q}_{i,t} $ $\triangleq Q_{\Delta_{i,t}}([o_{B,t}]_i+\bm{\eta}_{i,t})$.
If we define the reconstruction $[\bm{y}_{B,t}]_i \triangleq \bm{q}_{i,t}- \bm{\eta}_{i,t}$, it is shown \cite{tanaka2016rate} that $[\bm{y}_{B,t}]_i =  [o_{B,t}]_i +\bm{n}_{i,t}$, where $\bm{n}_{i,t}$ is the uniform random variable
\[
\bm{n}_{i,t} \sim \mathcal{U}\left(-\tfrac{\Delta_{i,t}}{2}, \tfrac{\Delta_{i,t}}{2}\right).
\]
Here, the noise $\bm{n}_{i,t}$ is a deterministic, one-to-one function of the dither $\bm{\eta}_{i,t}$ (see ~\cite[Appendix]{tanaka2016extended} for a detailed discussion), and thus
independent of $[o_{B,t}]_i$ and  $\bm{n}_{j,t}$ for all $j\neq i$.
We note that in order to compute the reconstructed signal $[y_{B,t}]_i =q_{i,t}-\eta_{i,t}$, the Seeker must use the same dither vector $\eta_{i,t}$ that was used by the Supporter during quantization. 
In practice, this is achieved through shared pseudo-randomness; that is, both robots use synchronized random number generators to produce identical dither realizations.
As a result, the signal received by the Seeker can be written as  
\begin{align*}
    \bm{y}_{B,t} = o_{B,t} + \bm{n}_t = \Theta_t C_{B,t} x + \bm{n}_t,
\end{align*}
where the noise vector $\bm{n}_t$ is defined by stacking the scalar uniform, independent, random variables $\bm{n}_{i,t}$. Since each $\bm{n}_{i,t}$ has $\mathbb{E}[\bm{n}_{i,t}] = 0 $ and $\mathbb{V}[\bm{n}_{i,t}]=\Delta_{i,t}^2/{12}$, we have that $\mathbb{E}[\bm{n}_t] = 0_{d_{\Theta_t}}$ and $ \displaystyle \mathbb{V}[\bm{n}_t]= N_t \triangleq \operatorname*{diag}_{1\leq i \leq d_{\Theta_t}} (\Delta_{i,t}^2/12)$.

In this paper, we develop a framework that allows us to jointly design the compression $\Theta_t$ and the quantization levels $\Delta_{i,t}$ resulting in the most informative signal $\bm{y}_{B,t}$ that is in compliance with the communication resource constraint. 
Importantly, the quantization $\Delta_{i,t}$ levels
are not necessarily equal across features; different components of the compressed map may be quantized at different levels depending on their task relevance, and consequently may incur different communication costs. 

In order to infer the received signal $\bm{y}_{B,t}$, the Seeker needs to know the adopted $\Theta_t$ and $N_{t}$. 
Communicating these variables
to the Seeker may drain the available communication resources and defeat the purpose of compression. 
In our proposed framework, we show in Section~\ref{sec:compression} that optimal compression and quantization levels can be determined independently by both the Supporter and the Seeker. 
As a result, there will be no need to send the designed $\Theta_t$ and $N_t$. 

\subsection{Estimation Module}
\label{subsec:estimation}

We assume that the Seeker begins with an initial belief $\bm{x}_0$ of the traversability vector $ x \in \mathbb{R}^d 
$ of the environment having mean $\mathbb{E}[\bm{x}_0] = \hat{x}_0 $ and covariance $\mathbb{V}[\bm{x}_0]=P_0$. At each time step, the Seeker computes a linear minimum variance estimate (LMVE) of $x$ using a recursive least squares (RLS) algorithm. This estimate is updated using both the Seeker's observations and the compressed representations received from the Supporter. Although the belief of the traversability vector and the measurements may not be jointly Gaussian, we adopt a \textit{linear estimator} that follows the structure of the Kalman filter. This estimator minimizes the mean squared error among all linear estimators and retains analytical tractability under non-Gaussian conditions. In particular, we do not assume a Gaussian posterior; instead, we maintain only the \textit{first two moments}, namely, the mean and the covariance, which evolve according to closed-form update rules as shown in \cite{psomiadis2025communication}.

Specifically, at time step $ t $, the Seeker maintains the belief $\bm{x}_t$ over $x$, characterized by the mean $ \hat{x}_t $ and the covariance $ P_t $. Upon receiving its own measurement $\bm{y}_{A,t}=y_{A,t}$, the Seeker updates its belief to $\bm{x}^+_t$ with updated mean  $\hat{x}^+_t$ and covariance $ P^+_t $, according to
\begin{subequations}
\label{eq:est_1}
\begin{align}
    \hat{x}^+_t &= \hat{x}_t + L_t^+(y_{A,t} - C_{A,t} \hat{x}_t), \\
    L_t^+ &= P_t C_{A,t}^\top(C_{A,t} P_t C_{A,t}^\top + \sigma_m^2 I)^{-1}, \\
    P^+_t &= (P_t^{-1} + \sigma_m^{-2} C_{A,t}^\top C_{A,t})^{-1}.
\end{align}
\end{subequations}
After updating its belief with its own observation, the Seeker incorporates the compressed observation $ \bm{y}_{B,t} = y_{B,t}$ received from the Supporter. This incorporation leads to a further update of the mean and covariance, resulting in the belief at the next time step $\bm{x}_{t+1}$, represented by $ \hat{x}_{t+1}$ and $ P_{t+1} $, respectively:
\begin{subequations}
\label{eq:est_2}
\begin{align}
    \hat{x}_{t+1} &= \hat{x}^+_t + L_t(y_{B,t} - \Theta_t C_{B,t} \hat{x}^+_t), \\
    L_t &= P^+_t C_{B,t}^\top \Theta_t^\top (\Theta_t C_{B,t} P^+_t C_{B,t}^\top \Theta_t^\top \!\!+ N_t)^{-1}, \\
    P_{t+1} &= ((P^+_t)^{-1} + C_{B,t}^\top \Theta_t^\top N_t^{-1} \Theta_t C_{B,t})^{-1}.
\end{align}
\end{subequations}
This update corresponds to the linear minimum variance estimator, which remains valid even if the underlying distributions deviate from Gaussianity. The recursive structure ensures that both the mean and covariance are propagated forward efficiently in closed form.

It is important to note that the belief covariances $ P^+_t$ and  $P_{t+1}$ are independent of the specific measurement realizations $ y_{A,t}$ and $y_{B,t}$. Rather, they are entirely determined by the observation models $C_{A,0:t}$ and $C_{B,0:t}$, which, in turn, are functions of the paths taken by the Seeker and the Supporter. As a result, both agents can compute the belief covariances independently knowing only the trajectories $ \textup{p}_{A,0:t} $ and $ \textup{p}_{B,0:t}$.
In the following sections, we leverage this property to show that the Seeker can compute the same optimal compression matrix $ \Theta_t^*$ and $N_t^*$ that the Supporter chooses. This ability of the Seeker removes the need for the Supporter to explicitly transmit either the compression matrix or the quantization parameters, further reducing communication overhead.

\subsection{Path-Planning Module}
\label{sect:planner}
Let $(V,E)$ represent the graph associated with the grid environment $M$, where $V=\{1, \dots, d\}$ denotes the set of vertices and $E$ the set of edges. 
Each vertex in $V$ corresponds to a specific cell in $M$ (with a slight abuse of notation, we will use $\textup{p}$ to denote both cell positions and graph vertices).
Two vertices are deemed connected if they are immediate neighbors in $M$. 
We assume that the cost incurred by traversing through a cell is proportional to its occupancy value (difficulty to traverse) plus a constant (movement penalty). 
The cost of traversing a vertex is therefore given by:
\begin{align}
\label{eq:cost_def}
    c_{\epsilon}(\textup{p})=\begin{cases}
        x(\textup{p}) + a & \text{if}\ \textup{p} \in \textup{P}_{\epsilon},\\
        d(\epsilon+a) & \text{otherwise},
    \end{cases}
\end{align}
where $x(\textup{p}) \in [0,1]$ is the traversability value of the cell at position $\textup{p}$, and $a$ is a constant cost for traversing a cell. 
$\textup{P}_{\epsilon} = \{ \textup{p} \in V : x(\textup{p}) \leq \epsilon\}$ designates the set of cells meeting a feasibility condition, where $\epsilon \in [0, 1]$ is a scalar that defines cell feasibility.
Let $\Pi_t$ denote the set of paths with the first element being the Seeker's current position $\textup{p}_{A,t}$ and the last element being its destination location $\textup{p}_{D}$. 
Then, the optimal path is given by:
\begin{align}
\label{eq:optimal_path}
    \pi_t^* = \argmin_{\pi \in \Pi_t} \sum_{\textup{p} \in \pi} c_{\epsilon}(\textup{p}).
\end{align}
When all elements of the path $\pi$ is in $\textup{P}_{\epsilon}$, it is referred to as an $\epsilon$-feasible path.
We set the second scale of \eqref{eq:cost_def} to $d(\epsilon+a)$, a value larger than the cost of any feasible path. This assignment effectively excludes infeasible vertices from consideration, unless no $\epsilon$-feasible path exists.
The optimal path $\pi^*$ can be computed using various graph search algorithms, such as Dijkstra, $A^*$, or $D^*$ \cite{stentz1994optimal}. 
$D^*$ is particularly advantageous, as it supports dynamic re-planning when the map is updated incrementally, based on new observations from onboard sensors. 
This dynamic re-planning capability eliminates the need to re-plan from scratch as new information becomes available.

The Seeker does not know the exact traversability values $[x]_i$ at time $t$. 
Instead, it computes an approximation of $c_{\epsilon}(\textup{p})$ based on the estimated values $\hat{x}^+_t$. 
However, the best estimate $\hat{x}^+_t$ computed by the RLS updates \eqref{eq:est_1} and \eqref{eq:est_2}  might not lie within the $[0, 1]$ interval. 
Thus, the Seeker uses the projected value 
\begin{align*}
     \tilde{x}_t = \max( \min(\hat{x}^+_{t}, 1), 0),
\end{align*}
in place of $x$ in \eqref{eq:cost_def} and in the definition of $\textup{P}_{\epsilon}$ to approximate $c_{\epsilon}(\textup{p})$. 
Then, it finds the shortest path by solving \eqref{eq:optimal_path} based on the approximated $c_{\epsilon}(\textup{p})$. 
The optimal path $\pi_t^*$ is then sent to the Supporter to design the optimal compression.

%%%%%%%%%%%%%%%%%%%%%%%%%%%%%%%%%%%%%%%%%%%%%%%%%%%%%%%%%%%%%%%%%%%%%%%%%%%%%%%%

\section{Communication-Aware Compression Design}
\label{sec:compression}

In this section, we introduce a function to formally quantify the importance (or informativeness) of map cells, used in the Utility Module, and a function that approximates the communication cost incurred in transmitting a compressed map. Based on these functions, we then formulate the compression design problem addressed by the Compression Scheme Module.

\subsection{Utility Module}

The utility of a compression is defined based on how much it can reduce the Seeker's uncertainty about the traversability values. 
In other words, the compression $\Theta_t$ is effective if it results in a small estimation covariance $P_{t+1}$. Not all cells have the same level of importance for path planning. 
For instance, the cells on the optimal path $\pi^*_t$ or in its vicinity are often more critical than those far from it.
This increased importance arises because observing cells near the Seeker's optimal path has  more influence on the Seeker's navigation, revealing if $\pi^*_t$ is indeed an $\epsilon$-feasible path.
% Also, these cells are more likely to be on the truly optimal path.

The work of \cite{psomiadis2024communication} suggested assigning weights to each cell within $M$ based on its proximity to the path $\pi_t^*$. 
They define the importance vector $w_t \in \mathbb{R}^d$, where each element is computed by the normalized Gaussian function
\begin{align}
\label{eq:w_def}
    w_t(\textup{p})= \max_{\textup{p}_{\pi} \in \pi^*_t} \exp{\left(-\frac{\|\textup{p}-\textup{p}_{\pi}\|^2}{2\sigma^2}\right)}, \quad \forall \textup{p} \in M, 
\end{align}
where $\sigma$ is a parameter characterizing the region of interest around the path.
This weight function can be interpreted as follows: For each cell $\textup{p} \in M$, we find the distance to the closest cell on the optimal path $\pi^*_t$. 
$w_t(\textup{p})$ is defined to be proportional to the exponential of the negative distance to $\pi^*$, with decay rate controlled by $\sigma$. The heat map  Fig.~\ref{fig:weight} demonstrates the weights of a cell for a given path (dotted cells).

\begin{figure}[tb]
    \centering        
{\includegraphics[trim = 5cm 3cm 2cm 1.5cm, clip, width=0.75\linewidth]{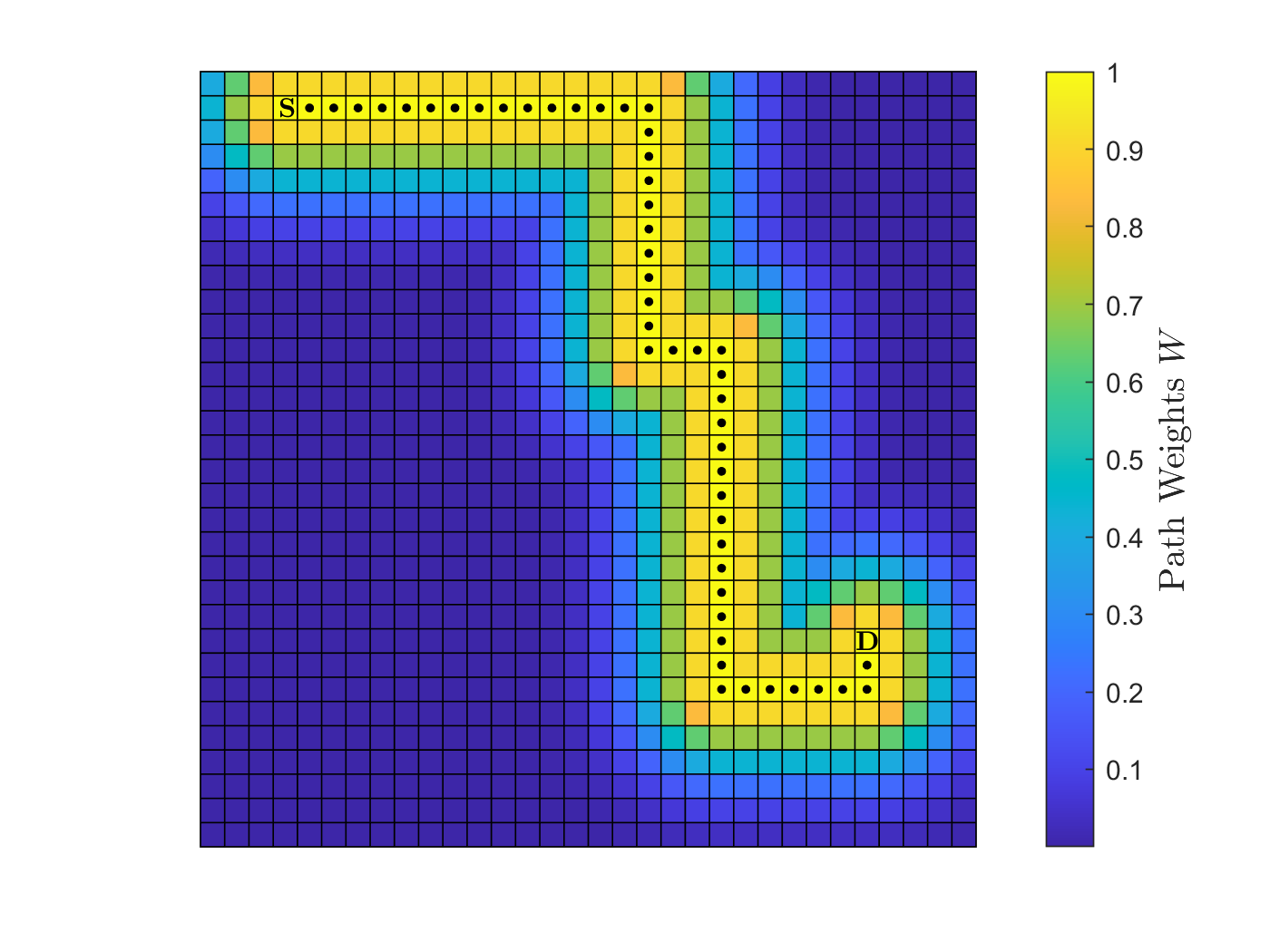}
    \caption{$\pi_t^*$ is marked with dots, where $\textup{S}$ denotes the starting cell and $\textup{D}$ denotes the destination cell. Path weights are computed using \eqref{eq:w_def}. The cells are color-coded based on their weights.}
    % \ali{Replace $\textup{G}$ to $\textup{D}$ in the picture. Any ideas for better quality?}}
    \label{fig:weight}}
\end{figure}

We follow \cite{psomiadis2024communication} and adopt the weight definition \eqref{eq:w_def}. 
Defining $W_t \triangleq \textup{diag}(w_t) \in \mathbb{D}_{d}$, we quantify the utility of the compression $\Theta_t$ as 
\begin{align}
    \textup{Utility}(\Theta_t) \triangleq -\mathrm{tr}(W_t P_{t+1}).
\end{align}
This metric can be interpreted such that a compression is better if it results in a smaller $P_{t+1}$ or, equivalently, higher (less negative) $\textup{Utility}(\Theta_t)$.

\subsection{Communication Cost}
To quantify the communication cost, we examine the structure of the entropy coding scheme applied to the quantized signal. As established in Section~\ref{subsec:com}, the compressed observation vector $o_{B,t}$ is a deterministic function of the global map, but from the Seeker's perspective, it is a random variable given by $\bm{o}_{B,t} = \Theta_t C_{B,t} \bm{x}_t^+$. The quantized signal is produced by applying dithered quantization to this random vector, resulting in a discrete random variable $\bm{q}_t$ with support
\[
\textup{supp}(\bm{q}_t)\triangleq \{ q_t \in \mathbb{R}^{d_{\Theta_t}} | [q_t]_i = k_{i,t} \Delta_{i,t}, k_{i,t} \in \mathbb{Z}\}.
\]
A coding scheme assigns to each quantized vector $q_t \in \textup{supp}(\bm{q}_t)$ a binary codeword from an alphabet $\mathcal{A}=\{0, 1, 00, 11, 01, 01, 000, \dots\}$. This is known as a variable-length, lossless source coding scheme \cite[Chapter~5]{cover1999elements}, where each quantized value is mapped to a unique binary representation such that the original value can be perfectly reconstructed from the codeword.
For a given coding scheme and $\bm{q}_t$, the length of the resulting codeword is a random variable $\bm{l}_t$. The expected number of bits required to send the compressed representation over the ECDQ scheme introduced in Section~\ref{subsec:com} is given by the expected codeword length:
\begin{align}
\mathbb{E}[\bm{l}_t] = \sum_{q_t \in \textup{supp}(\bm{q}_t)} \mathbb{P}(\bm{q}_t = q_t)  l(q_t),
\end{align}
where $l(q_t)$ denotes the length of the codeword assigned to $q_t$.

Ideally, we want the coding scheme at time-step $t$ to be adapted to the conditional distribution of $\bm{q}_{t}$ given $\bm{\eta}_t$ so as to minimize  $\mathbb{E}[\bm{l_t}]$.  Intuitively, this objective can be achieved by mapping more probable values of $q_{t}$ to binary code-words with shorter length.  However, designing the optimal coding scheme and computing the exact value of $\mathbb{E}[\bm{l}_t]$ is a non-trivial task in \textit{Information Theory} \cite{cover1999elements}. 
The following proposition establishes upper and lower bounds on the minimum expected codeword length $\mathbb{E}[\bm{l}_t]$ that can be attained by an optimal coding scheme. These bounds are expressed in terms of the mutual information between  $\bm{o}_{B,t}$ and $\bm{y}_{B,t}$, and indicate that $ I(\bm{o}_{B,t}; \bm{y}_{B,t})$ serves as a close approximation to the optimal expected codeword length.
\begin{proposition}
For an optimal, variable-length, lossless source coding scheme, the expected codeword length satisfies 
\begin{align} \label{eq:length}
    I(\bm{o}_{B,t}; \bm{y}_{B,t}) \leq \mathbb{E}[\bm{l}_t] < I(\bm{o}_{B,t}; \bm{y}_{B,t}) + 1.
\end{align}
\end{proposition}
\begin{proof}
    See \cite{tanaka2016rate}[Theorem~2 and Lemma~1] for the detailed proof.
\end{proof}

Note that $\bm{o}_{B,t}$ and $ \bm{y}_{B,t}$ are non-Gaussian, and hence evaluating the mutual information $I(\bm{o}_{B,t}; \bm{y}_{B,t})$ is rather difficult. 
One approach to evaluate this mutual information was proposed in \cite{silva2009unified} by defining a \textit{jointly Gaussian} version of these variables. Specifically, we define the Gaussian version of  $\bm{n_t}$ with the same mean and covariance as 
\begin{align} \label{eq:nG_def}
   \bm{n}^G_t \sim \mathcal{N}\left(0_{d_{\Theta_t}}, \operatorname*{diag}_{1\leq i \leq d_{\Theta_t}} (\Delta_{i,t}^2/12)\right). 
\end{align}
This Gaussian noise model allows us to bound the mutual information as follows.
\begin{proposition} \label{prop}
With $\bm{n}^G_t$ as in \eqref{eq:nG_def}, we have
\begin{align} \nonumber
   I(\bm{o}_{B,t} ; \bm{y}_{B,t})
   &-I(\bm{o}^G_{B,t} ; \bm{y}^G_{B,t}) \\
   \label{eq:KL_diff}
   &= \mathrm{KL}(\bm{n}_{t} \| \bm{n}^G_{t}) -\mathrm{KL}(\bm{y}_{t} \| \bm{y}^G_{t}),
\end{align}
where $\mathrm{KL}(\bm{n}_{t} \| \bm{n}^G_{t}) = \frac{d_{\Theta_t}}{2} \log \frac{2\pi e}{12}$. Furthermore, since $\mathrm{KL}(\bm{y}_{t} \| \bm{y}^G_{t}) \geq 0$, we have
\begin{align} \label{eq:upper_bound}
   I(\bm{o}_{B,t} ; \bm{y}_{B,t}) 
   &\leq I(\bm{o}^G_{B,t} ; \bm{y}^G_{B,t}) + \frac{d_{\Theta_t}}{2} \log \frac{2\pi e}{12}.
\end{align}
If we assume
\begin{align} \label{eq:Kl_rel}
    \mathrm{KL}(\bm{y}_{t} \| \bm{y}^G_{t}) \leq \mathrm{KL}(\bm{n}_{t} \| \bm{n}^G_{t}),
\end{align}
then, 
\begin{align} \label{eq:kl_res}
    I(\bm{o}^G_{B,t} ; \bm{y}^G_{B,t}) \leq I(\bm{o}_{B,t} ; \bm{y}_{B,t}).
\end{align}
\end{proposition}

\begin{proof}
See Appendix~\ref{appendix_zero}. The proof builds on techniques from \cite[Lemma C.1]{silva2009unified},  with a scalar version presented in \cite[Lemma~2]{jung2023optimized}.
\end{proof}

Condition \eqref{eq:Kl_rel} suggests that if the measurement distribution is closer to Gaussian than the uniform noise, the Gaussian approximation underestimates the true mutual information. By leveraging the relation between $\mathbb{E}[\bm{l}_t]$ and $I(\bm{o}_{B,t}; \bm{y}_{B,t})$ from \eqref{eq:length}, and the relation between $I(\bm{o}_{B,t}; \bm{y}_{B,t})$ and $I(\bm{o}^G_{B,t}; \bm{y}^G_{B,t})$ from \eqref{eq:upper_bound} and \eqref{eq:kl_res}, we adopt $I(\bm{o}^G_{B,t}; \bm{y}^G_{B,t})$ as the communication cost and refer to it as bit-rate, which can be computed as \begin{align} \nonumber
    &I(\bm{o}^G_{B,t} ; \bm{y}^G_{B,t}) = h(\bm{y}^G_{B,t})- h(\bm{y}^G_{B,t}|\bm{o}^G_{B,t})\\
    \nonumber 
    &=h(\bm{y}_{B,t}^G)- h(\bm{o}_{B,t}^G+ \bm{n_t}^G|\bm{o}_{B,t}) \\
    \nonumber
    &=h(\Theta_t C_{B,t} \bm{x}^{+G}_t+ \bm{n_t}^G)-h(\bm{n_t}^G)\\ \label{eq:mutal_first}
    &=\!\frac{1}{2}\!\log\!\det(\Theta_t C_{B,t} P_t^+ C_{B,t}^\top \Theta_t^\top\!\!\!+\!N_{t}\!) \!-\! \frac{1}{2}\! \log\!\det(N_{t}).
\end{align}

\subsection{Compression Scheme Module}

The encoder's role is to design the optimal compression strategy of the Supporter, focusing on both navigation and communication aspects.
This design involves considering the path weights $w_t$, as defined in \eqref{eq:w_def}, and penalizing compressions based on their required transmission bit-rates.
At each time step $t$, the Supporter optimally designs the compression and quantization by minimizing the weighted sum  
\begin{align} \label{eq:alpha_def}
   -\textup{Utility}(\Theta_t) + \alpha I(\bm{o}^G_{B,t} ; \bm{y}^G_{B,t}), 
\end{align}
where $\alpha$ is the weight on the bit-rate. 
More precisely, the Supporter solves the following optimization problem:
\begin{subequations}
\label{eq:opt_main}
    \begin{align} \nonumber
     \min \ \  &\mathrm{tr}(W_t P_{t+1})+\frac{\alpha}{2} \log\! \det(\Theta_t C_{B,t} P_t^+ C_{B,t}^\top  \Theta_t^\top\!\!+ N_{t})\\
     &\qquad \qquad \quad -  \frac{\alpha}{2} \log\!\det(N_{t}),
     \label{eq:nonconvex_obj}
     \\
    \text{s.t.,} \ \  & P_{t+1}\! =\! ((P_{t}^{+})^{-1}\!+ C_{B,t}^\top  \Theta_t^\top  N_{t}^{-1} \Theta_t C_{B,t} )^{-1}, \label{eq:nonlinear_equality}
\end{align}
\end{subequations}
where $P_{t+1} \in \mathbb{S}_{+}^{d}, N_{t} \in \mathbb{D}_{d_{\Theta_t}}$, $\Theta_t \in \mathbb{R}^{d_{\Theta_t}\times d_{B}}$ and $d_{\Theta_t} \in \{0, 1, \dots, d_B\}$. The
Seeker also solves \eqref{eq:opt_main}, prior to receiving $y_{B,t}$, to infer what $\Theta_t^*$ the Supporter will use, and then executes the estimation step \eqref{eq:est_2} after receiving the compressed map sent by the Supporter.  
Problem \eqref{eq:opt_main} is non-convex due to the nonlinear equality constraint \eqref{eq:nonlinear_equality} and nonlinear objective function \eqref{eq:nonconvex_obj}. Moreover, the integer-valued size $d_{\Theta_t}$ is also an optimization variable. 
In the next section, we propose an efficient method to solve \eqref{eq:opt_main} to global optimality by reformulating this mixed-integer non-convex problem as a convex log-det minimization problem.

%%%%%%%%%%%%%%%%%%%%%%%%%%%%%%%%%%%%%%%%%%%%%%%%%%%%%%%%%%%%%%%%%%%%

\section{Main Results}
To solve the non-convex Problem~\eqref{eq:opt_main}, we suggest reformulating it in terms of the reordered state $x_{R,t}=C_{R,t} x$ introduced in Section~\ref{subsec:obs}. 
Specifically, we consider the Seeker's best estimate of $x_{R,t}$ at time $t$ prior to receiving $ \bm{y}_{B,t} =y_{B,t}$ by $\hat{\bm{x}}^+_{R,t} $ with $\mathbb{E}[\hat{\bm{x}}^+_{r,t}]= \hat{x}^+_{R,t} = C_{R,t} \hat{x}_{t}^+$ and 
$ \mathbb{V}[\hat{\bm{x}}^+_{R,t}] = P^+_{R,t} = C_{R,t} \hat{P}_{t}^+ C_{R,t}^\top$, and after receiving $ \bm{y}_{B,t} =y_{B,t}$ by $\hat{\bm{x}}_{R,t+1} $ with $ \mathbb{E}[\hat{\bm{x}}_{R,t+1}] = \hat{x}_{R,t+1} = C_{R,t} \hat{x}_{t+1}$ and $ \mathbb{V}[\hat{\bm{x}}_{R,t+1}]=  P_{R,t+1} = C_{R,t} \hat{x}_{t+1} C_{R,t}^\top$, respectively. 
We decompose the error covariances to
\begin{align*}
    P_{R,t}^+= \begin{bmatrix}
        P_{BB,t}^+ & P_{BO,t}^+\\
        P_{OB,t}^+ & P_{OO,t}^+
    \end{bmatrix}, \ 
    P_{R,t+1}= \begin{bmatrix}
        P_{BB,t+1} & P_{BO,t+1}\\
        P_{DO,t+1} & P_{OO,t+1}
    \end{bmatrix},
\end{align*} 
where 
$P^{+}_{BB,t}=C_{B,t} P^+_{t} C_{B,t}^\top$,
$P^{+}_{OO,t} = C_{O,t} P^+_{t} C_{O,t}^\top$, 
and $P^{+}_{BO,t} = (P^{+}_{DO,t})^\top = C_{O,t} P^+_{t} C_{D,t}^\top$. The elements of $P_{R,t+1}$ are defined similarly. 
We accordingly define the diagonal matrix $W_{R,t}\triangleq C_{R,t} W_t C_{R,t}^\top$, which can be broken into $W_{R,t} = \mathrm{diag}(W_{BB,t}, W_{OO,t})$.

\begin{theorem}
 \label{theo:one}
 The following claims hold.
 \begin{enumerate}[label=\roman*)]
     \item $P_{BB,t+1}$ and $P_{OO,t+1}$ can be computed as
    \begin{align} \label{eq:p_bb}
        P_{BB,t+1} &= ((P_{BB,t}^{+})^{-1}\!+ \Theta_t^\top N_{t}^{-1} \Theta_t)^{-1},\\ \label{eq:p_dd}
         P_{OO,t+1} &= S_{t}+ Q_{t} P_{BB,t+1} Q_{t}^\top,
    \end{align}
    where $Q_{t} \triangleq P_{OB,t}^+ (P_{BB,t}^+)^{-1}$, and $S_{t} \triangleq P_{OO,t}^{+}- Q_{t} P_{BB,t}^{+} Q_{t}^\top$.
    \item The utility  
    $\mathrm{tr}(W_t P_{t+1})$ can be computed as
    \begin{align} \nonumber
        \mathrm{tr}(&W_t P_{t+1}) = \mathrm{tr}(W_{R,t} P_{R,t+1})\\ \label{eq:trace}
        &=\mathrm{tr}(\tilde{W}_{BB,t} P_{BB,t+1}) + \mathrm{tr}(W_{OO,t}S_{t}),
        \end{align} 
        where $\tilde{W}_{BB,t} \triangleq W_{BB,t} + Q_{t}^\top W_{OO,t} Q_{t}$.
    \item The bit-rate $I(\bm{o}^G_{B,t} ; \bm{y}^G_{B,t})$ can be computed as 
     \begin{align*}
       \hspace{-0.3cm} I(\bm{o}^G_{B,t} ; \bm{y}^G_{B,t}) \!=\! \frac{1}{2} \log\!\det( P_{BB,t}^+)\!-\!\frac{1}{2} \log\!\det( P_{BB,t+1}).
    \end{align*}
 \end{enumerate}
\end{theorem}
\begin{proof}
    See Appendix~\ref{appendix_one}.
\end{proof}

Using the results of Theorem~\ref{theo:one}, we may rewrite \eqref{eq:opt_main} in terms of the variables $P_{BB,t+1}$, $\Theta_t$, $N_{t}$ and $d_{\Theta_t}$ as
\begin{subequations}
\label{eq:opt_pbb}
    \begin{align}\label{eq:opt_pbb_a}
    \min \ \ &\mathrm{tr}(\tilde{W}_{BB,t} P_{BB,t+1})\!-\!\frac{\alpha}{2} \!\log\!\det(P_{BB,t+1}),\\ \label{eq:opt_pbb_b}
    \text{s.t.,} \ \  & P_{BB,t+1}\! =\! ((P_{BB,t}^{+})^{-1}\!+  \Theta_t^\top N_{t}^{-1} \Theta_t)^{-1},
\end{align}
\end{subequations}
where the constants 
$\mathrm{tr}(W_{OO,t}S_{t})$ and 
$\frac{\alpha}{2} \log \det( P^+_{BB,t})$ are omitted from the objective function \eqref{eq:opt_pbb_a}.
A mathematically similar problem to \eqref{eq:opt_pbb} appeared in the seminal paper \cite{tanaka2016semidefinite}. 
Following the method proposed in \cite{tanaka2016semidefinite}, we transform \eqref{eq:opt_pbb} into an optimization problem in terms of $P_{BB,t+1}$ only.  
This reformulation is achieved by replacing the nonlinear equality constraint \eqref{eq:opt_pbb_b} with a linear inequality constraint $ P_{BB,t} \preceq P_{BB,t-1}^+$. 
This replacement eliminates all the variables $\Theta_t$, $N_{t}$ and $d_{\Theta_t}$ from \eqref{eq:opt_pbb} yielding
\begin{subequations}
\label{eq:opt_convex}
    \begin{align}
    \min_{P_{BB,t+1}} \ \  &\mathrm{tr}(\tilde{W}_{BB,t} P_{BB,t+1})\!-\!\frac{\alpha}{2}\! \log\!\det(P_{BB,t+1}),\\ \label{eq:opt_con_cons}
    \text{s.t.,} \quad &P_{BB,t+1} \preceq P^+_{BB,t}.
\end{align}
\end{subequations}
Equations \eqref{eq:opt_pbb} and \eqref{eq:opt_convex} are mathematically equivalent, since eliminated variables can be constructed from any $P_{BB,t+1}$ satisfying \eqref{eq:opt_pbb_b} through
\begin{align}
\label{eq:theta_calc}
\Theta_t^\top  N_{t}^{-1} \Theta_t = P_{BB,t+1}^{-1} - (P^+_{BB,t})^{-1},
\end{align}
as follows. First, we compute the eigenvalue decomposition of the RHS of \eqref{eq:theta_calc} as $V_t^\top \Lambda_t V_t$. Due to the constraint \eqref{eq:opt_con_cons}, all eigenvalues are non-negative. Then, we decompose the diagonal matrix $\Lambda_t$ into two parts based on the strictly positive eigenvalues and zero eigenvalues as $\Lambda_t=\textup{blkdiag}(\Lambda_{t,+}\!\succ 0, 0)$, and, accordingly, decompose $V_t$ as $V_t=[V_{t,+}, V_{t, \mathrm{zero}}]$. It can be shown that $\Theta_t=V_{t,+}$, $\Sigma_{N,t}= \Lambda_{t,+}^{-1}$, and $d_{\Theta_t}= \mathrm{rank}(\Lambda_{t,+})$ (i.e., the number of strictly positive eigenvalues of $\Lambda_t$) satisfy \eqref{eq:theta_calc} and constitute feasible solutions for \eqref{eq:opt_pbb}.  To summarize this process, we define the reduced singular decomposition operator for any symmetric positive semi-definite matrix $A$ as $\textup{RSD}(A)= (V_{+}, \Lambda_{+})$, where $A=V^\top \Lambda V$, and $V_{+}$ and $\lambda_{+}$ retain only the components associated with strictly positive eigenvalues.  Then, Equation \eqref{eq:theta_calc} yields 
\begin{align*}
   (\Theta_t^\top,  N_{t}^{-1})  = \textup{RSD} (P_{BB,t+1}^{-1} - (P^+_{BB,t})^{-1}),
\end{align*}
with $d_{\Theta_t}= \mathrm{rank}(N_t)$. 

Problem \eqref{eq:opt_convex} is a determinant maximization problem subject to linear matrix inequality constraints.
As shown by \cite{vandenberghe1998determinant},  such problems are convex and can be efficiently solved using off-the-shelf semi-definite programming solvers. After finding $P_{BB,t+1}^*$, the minimizer of \eqref{eq:opt_convex}, the optimal compression $\Theta_t^*$ and optimal quantization $N_{t}^*$ can be constructed as explained above. It is worth noting that the optimal solutions are not unique. For instance, for any nonzero scalar $s$, the pair $(\Theta'_t=s\Theta_t^*, N'_{t}=s^2 N_{t}^*)$ is also optimal because it satisfies
\[
\Theta'^\top (N'_{t})^{-1} \Theta'= {\Theta_t^*}^\top (N_{t}^*)^{-1} \Theta_t^*.
\]
This non-uniqueness is intuitive, as the measurement $ \bm{y}_{B,t} = \Theta_t o_{B,t} +\bm{n}_t= y_{B,t}$ contains the same information as the scaled version $\bm{y}'_{B,t} = s\Theta_t o_{B,t} +s\bm{n}_t= y'_{B,t} = s y_{B,t}$.
Likewise, for any unitary matrix $U \in \mathbb{R}^{d_{\Theta_t} \times d_{\Theta_t}}$, the pair $(\Theta''_t = U \Theta_t^*, N''_{t}=U N_{t}^* U^\top)$ is also optimal. 

Although Problem~\eqref{eq:opt_convex} is convex and can, in principle, be solved efficiently, its high-dimensional nature, due to the size of the local map, may lead to computational inefficiencies in practice. Notably,
Problem \eqref{eq:opt_convex} is an instance of the standard Gaussian rate-distortion problem \cite{cover1999elements}.
It is well-known that this problem recovers the \textit{reverse water-filling} solution as explained in the following theorem.  

\begin{theorem}
\label{theo:two}
Denote the eigen-decomposition of $\tilde{W}_{BB,t}^{1/2}$ $P^+_{BB,t} \tilde{W}_{BB,t}^{1/2}$ by $U_t^\top D_t U_t$, where $D_t =\mathrm{diag}_{1\leq i\leq d_B}( \sigma_{i,t}^2)$ and $U_t$ is an unitary matrix $U_t U_t^\top = U_t^\top U_t=I$.
Then, the optimizer of \eqref{eq:opt_convex} is
 \begin{align} \label{eq:P_opt}
    P_{BB,t+1}^* 
   &= \tilde{W}_{BB,t}^{-1/2} U_t^\top \Phi_t  U_t \tilde{W}_{BB,t}^{-1/2},
\end{align} 
where,
\begin{align*}
    \Phi_t=\mathrm{diag}_{1\leq i \leq d_B} (\min\{ \frac{\alpha}{2}, \sigma_{i,t}^2\}).
\end{align*}
Moreover,
\begin{align} \label{eq:M_def}
   \!\!M_t &\!\triangleq\! {(P_{BB,t+1}^*)}^{-1}\!\!\!-\! (P^+_{BB,t})^{-1}  
   \!\!=\!\!\tilde{W}_{BB,t}^{1/2} U_t^\top \Sigma_t  U_t \tilde{W}_{BB,t}^{1/2},
\end{align}
where,
\begin{align}
\label{eq:sigma_opt}
 \Sigma_t = \operatorname*{diag}_{1 \leq i \leq d_{B}}(\max\{0, \frac{2}{\alpha}- \frac{1}{\sigma_{i,t}^2}\}).
\end{align}
The optimum compression and quantization levels can be computed as 
\begin{align} \label{eq:opt_theta}
    (\Theta_t^*, (N_{t}^*)^{-1}) 
   &=  \textup{RSD}(M_t).
   % \textup{RSD}( \tilde{W}_{BB,t}^{1/2} U_t^\top \Sigma_t  U_t \tilde{W}_{BB,t}^{1/2}).
\end{align} 

\end{theorem}
\begin{proof}
    The proof is provided in \cite[Subsection~IV.D]{tanaka2016semidefinite}. For completeness, the proof is summarized in Appendix~\ref{appendix_two} .
\end{proof}

By defining the set of indices of strictly positive diagonal elements of $\Sigma_t$ by $J_t \triangleq \{1\leq i \leq d_B: \alpha <  2\sigma_{i,t} ^2\}$, equations \eqref{eq:opt_theta} and \eqref{eq:sigma_opt} imply that $d_{\Theta_t^*}=$ dim$(\bm{y}^*_{B,t})= |J_t|$. 
This implication shows that the proposed formulation admits a low-dimensional solution, which is in full agreement with our intuition about compression. 
Moreover,  $d_{\Theta^*_t}$ monotonically decreases as the weight on the bit-rate, i.e., $\alpha$ in Equation \eqref{eq:alpha_def}, increases. 
Hence, the proposed bit-rate metric acts as an approximation of the cardinality of the used compressed representation. 

Using the closed-form expression \eqref{eq:opt_theta}, one may not need to solve the optimization problem \eqref{eq:opt_pbb}, and the optimal compression and quantization can be obtained from the strictly positive eigenvalues of the RHS of \eqref{eq:opt_theta} and the associated eigenvectors. We emphasize that not all elements of $N^*_{t}$ are equal.   
This non-uniformity highlights the flexibility of our formulation, wherein the bit-rate is optimally allocated across the components of the compression: more important components are preserved with higher precision, while less critical ones are compressed more aggressively.

\subsection{Noise in Supporter's Observation}
\label{subsec:noisy}
In this section, we consider the case where the Supporter's observation contains Gaussian noise as
\begin{align} \label{eq:noisy_obs}
\bm{o}_{B,t}= C_{B,t} x + \bm{\mu}_t, \quad \bm{\mu_t} \sim \mathcal{N}(0, V_t \succ 0).
% \sigma_{\mu}^2 I_{d_B}). 
\end{align}
Under this condition, the Seeker receives a measurement of $\bm{y}_{B,t}= \Theta_t C_{B,t} x + \Theta_t \bm{\mu}_t+ \bm{n_t}$.
All the results derived for the noise-less measurement case in Theorem~\ref{theo:one} can be repeated for the noisy measurement case by replacing $\bm{n}_t$ with 
% $\tilde{\bm{n}}_t \triangleq
$\Theta_t \bm{\mu}_t+ \bm{n_t}$, and, consequently, $N_{t}$ by 
% $\tilde{N}_t \triangleq$
$\Theta_t V_t \Theta_t^\top +N_{t}$. However, the optimal compression design problem \eqref{eq:opt_pbb}  becomes
\begin{subequations}
\label{eq:opt_noisy}
    \begin{align}
    \min \ \ &\mathrm{tr}(\tilde{W}_{BB,t} P_{BB,t+1})\!-\!\frac{\alpha}{2} \!\log\!\det(P_{BB,t+1}),\\ 
    % \label{eq:opt_noisy_b}
    \nonumber
    \text{s.t.,} \ \  & P_{BB,t+1}\! =\! ((P_{BB,t}^{+})^{-1} +\\ \label{eq:opt_noisy_c}
    & \qquad \qquad \quad \Theta_t^\top ( \Theta_t V_t \Theta_t^\top +N_{t}) ^{-1}\Theta_t)^{-1},
\end{align}
\end{subequations}
where the variables are $P_{t+1} \in \mathbb{S}_{+}^{d}, N_{t} \in \mathbb{D}_{d_{\Theta_t}}$, $\Theta_t \in \mathbb{R}^{d_{\Theta_t}\times d_{B}}$, and $d_{\Theta_t} \in \{0, 1, \dots, d_B\}$. 
Problem \eqref{eq:opt_noisy} cannot be transformed to Problem \eqref{eq:opt_convex}, in general.
The following proposition provides a condition under which \eqref{eq:opt_noisy_c} can be relaxed to \eqref{eq:opt_con_cons}, thereby enabling the transformation to \eqref{eq:opt_convex}.

\begin{proposition}
    Let $P_{BB,t+1}^*$ be the optimizer of \eqref{eq:opt_convex} as given in \eqref{eq:P_opt} in Theorem~\ref{theo:two}, and let $M_t$ be defined accordingly as \eqref{eq:M_def}. Let
    \begin{align} \label{eq:condition}
        M_t \prec V_t^{-1}. 
    \end{align}
     Then, $P_{BB, t+1}^*$ is the optimizer of \eqref{eq:opt_noisy}, and the corresponding compression and quantization parameters are given by
     \[ (\Theta_t^*, (N_{t}^*)^{-1}) 
   =  \textup{RSD}((V_t - V_t M_t V_t)^{-1}-V_t^{-1}).
\] 
\end{proposition}
\begin{proof}
    From \eqref{eq:opt_noisy_c}, we have $M_t = \Theta_t^\top ( \Theta_t V_t \Theta_t^\top +N_{t}) ^{-1}\Theta_t$. By pre- and post-multiplying both sides with $V_t$, we obtain
    \begin{align*}
        V_t M_t V_t &= V_t \Theta_t^\top ( \Theta_t V_t \Theta_t^\top +N_{t}) ^{-1}\Theta_t V_t \\
        &= V_t - (V_t^{-1} + \Theta_t^\top N^{-1} \Theta_t)^{-1},
    \end{align*}
    where we have used the Woodbury matrix inversion Lemma \cite{woodbury1950inverting}. Rearranging terms under the condition $V_t\succ V_t M_t V_t$ (or, equivalently, $M_t \prec V_t^{-1}$), gives
    \begin{align*}
        \Theta_t^\top N^{-1} \Theta_t= (V_t - V_t M_t V_t)^{-1}-V_t^{-1},
    \end{align*}
    which completes the proof.
\end{proof}

The matrix $V_t^{-1}$ represents the information content of the original observation model \eqref{eq:noisy_obs}, and thus characterizes the maximum information that can be extracted by any compression \eqref{eq:noisy_obs}. Hence, if $M_t$, which quantifies the information content of the optimally compressed signal obtained from the noise-free formulation (i.e., Problem~\eqref{eq:opt_convex}), satisfies condition \eqref{eq:condition}, then it is achievable under the noisy model, and $M_t$ is a valid optimizer of Problem~\eqref{eq:opt_noisy}.

\subsection{Planning based on Estimated Values}
As discussed in Section~\ref{sect:planner}, path-planning is performed based on the estimated occupancy values $\hat{x}_t$ projected on the interval $[0, 1]$. However, planning based on these estimates can be problematic when the utilized compression $\bm{y} = \Theta x_B + \bm{n}$ with $\bm{n} \sim \mathcal{N}(0, N)$ has low signal-to-noise ratio (SNR) components. We define the SNR matrix as $\textup{SNR} \triangleq {\Theta}^{\top} {N}^{-1} \Theta$, where small eigenvalues correspond to low-SNR components.

%%%%%%%%%%%%%%%%%%%%%%%%%%%
\begin{figure}
\begin{subfigure}[b]{0.232\textwidth}
    \centering
    \includegraphics[trim = 8cm 3cm 6cm 2cm, clip, width=\textwidth]{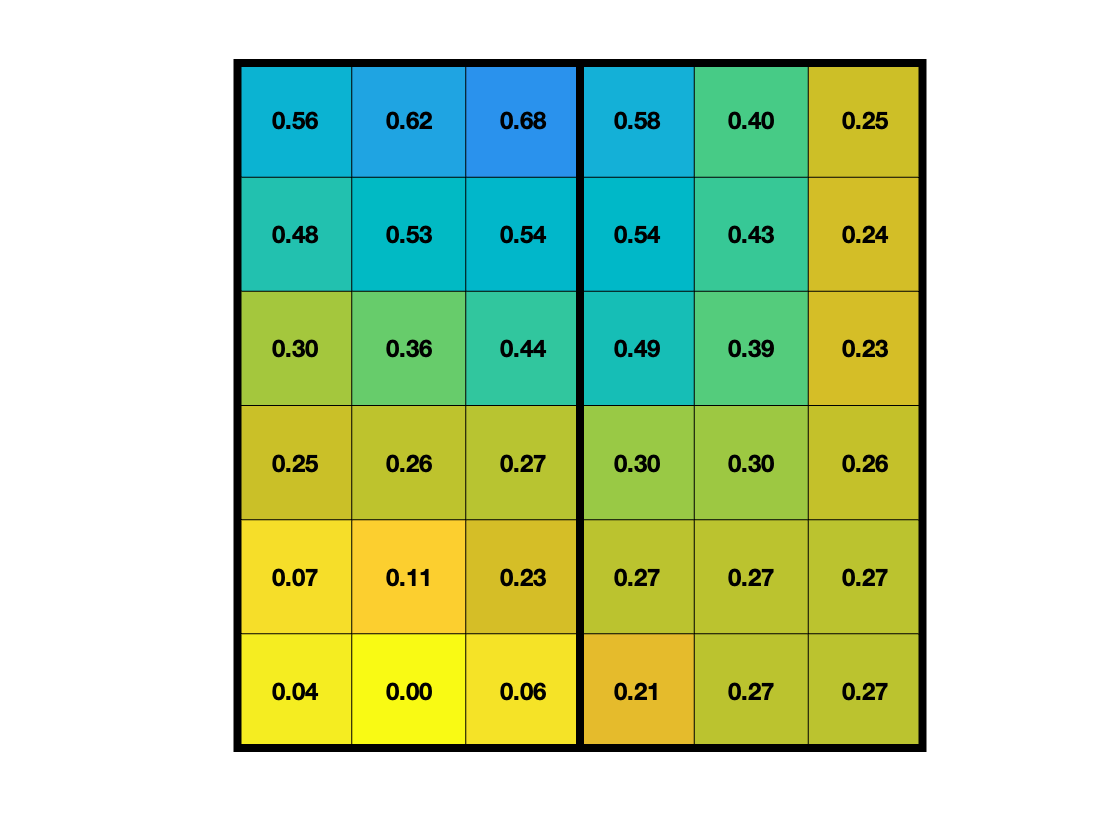}
    \caption{Real occupancy map: $x$. }
    \label{fig:snr_actual}
\end{subfigure}\ \
\begin{subfigure}[b]{0.232\textwidth}
    \centering
\includegraphics[trim = 8cm 3cm 6cm 2cm, clip, width=\textwidth]{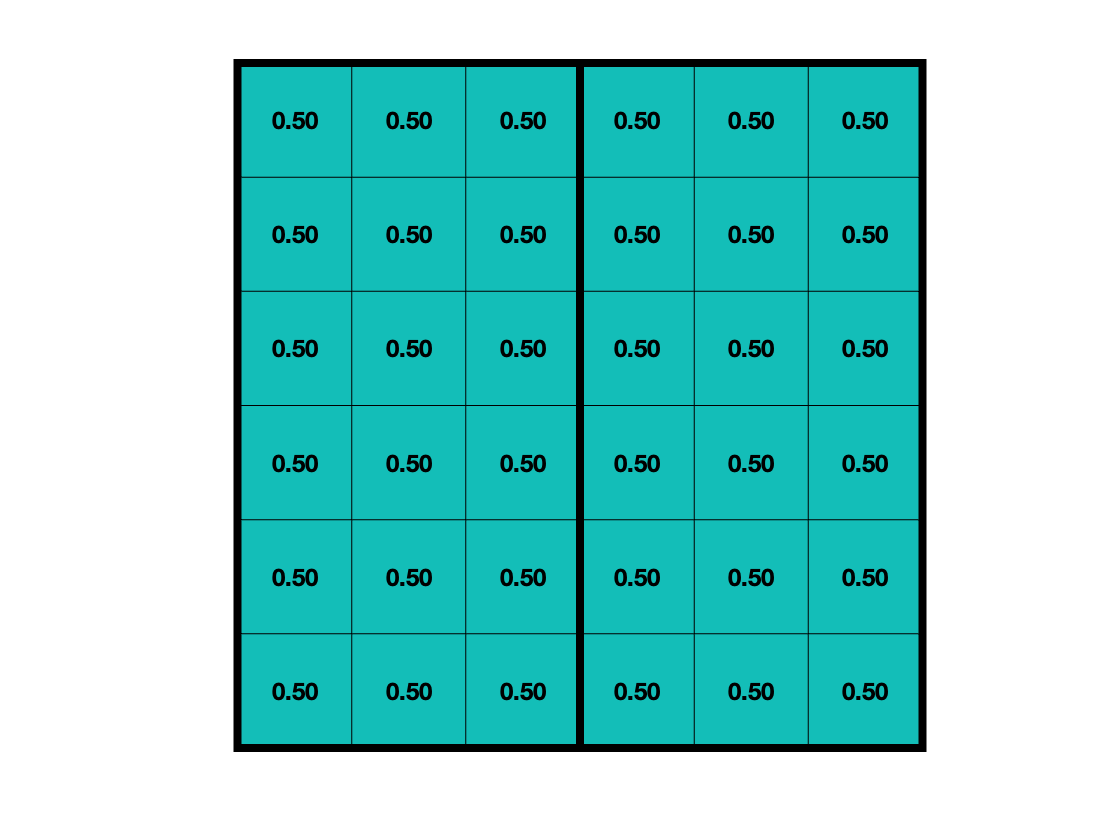}
    \caption{Initial belief: $\tilde{x}_0 = 0.5 \bm{1}_{49}$.}
    \label{fig:snr_initial}
\end{subfigure}

\vspace{0.5em}

\begin{subfigure}[b]{0.232\textwidth}
    \centering \includegraphics[trim = 8cm 3cm 6cm 2cm, clip, width=\textwidth]{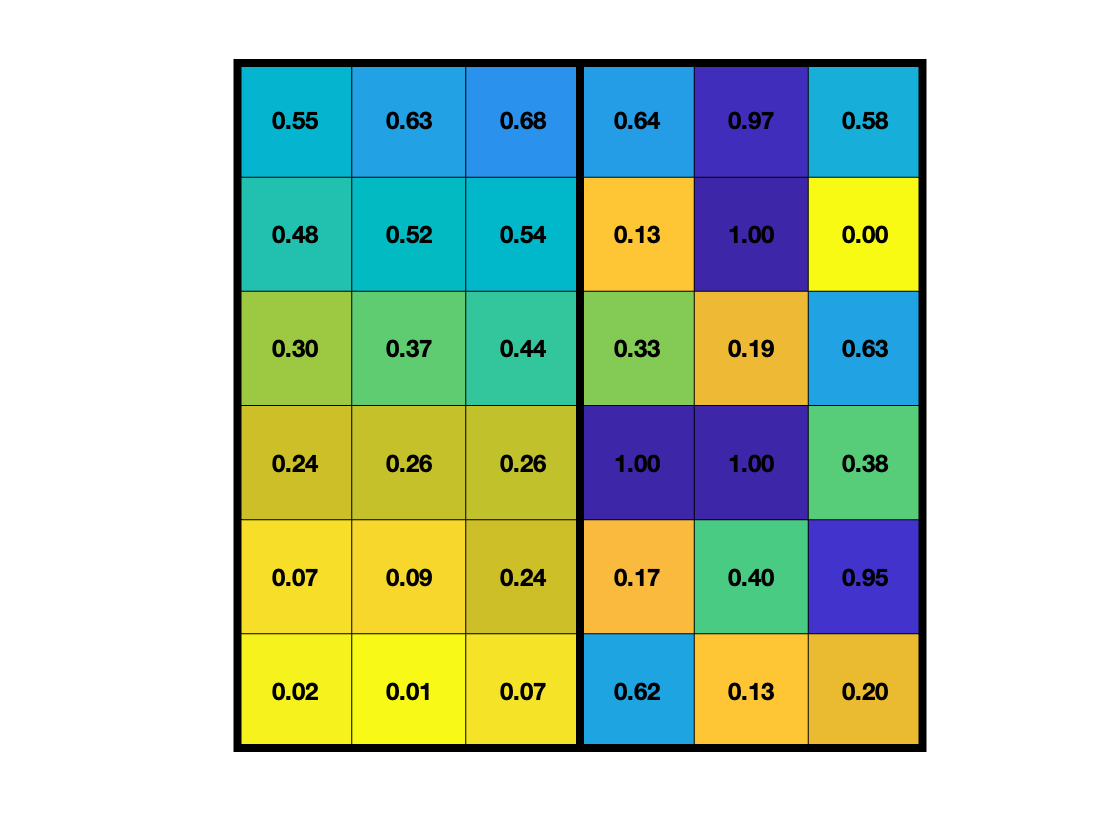}
    \caption{New belief via SRD.}
    \label{fig:snr_projeted}
\end{subfigure}\ \ 
\begin{subfigure}[b]{0.232\textwidth}
    \centering
\includegraphics[trim = 8cm 3cm 6cm 2cm, clip, width=\textwidth]{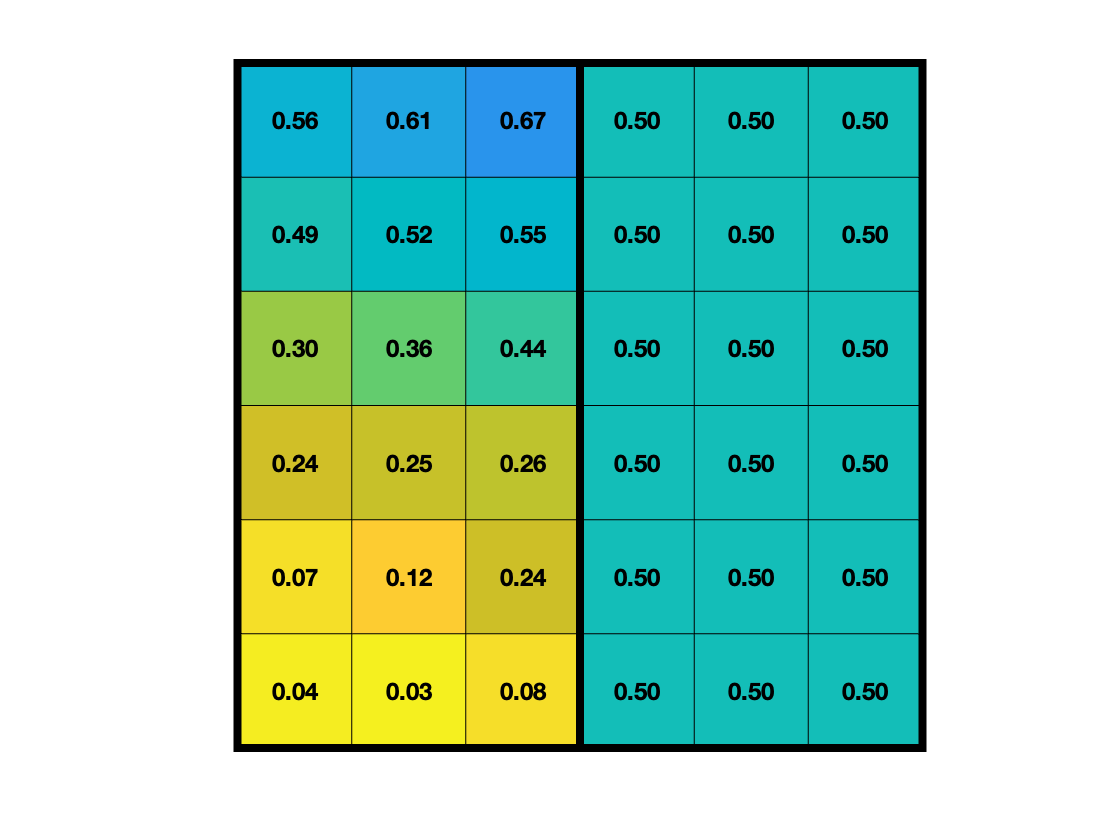}
    \caption{New belief via MSRD.}
    \label{fig:snr_modified}
\end{subfigure}
\caption{Occupancy values of a $6 \times 6$ grid-world map. (a) Real occupancy values of the map (b) Initial estimated occupancy values (c) Estimated and projected occupancy values after receiving mixed of low and high SNR compressions (d) Estimated and projected occupancy values after filtering the low SNR compressions via using the operator MSRD.}
\label{fig:snr}
\end{figure}

%%%%%%%%%%%%%%%%%%%%%%%%%%%

To illustrate the issue, consider the example in Fig.~\ref{fig:snr}.
Fig.~\ref{fig:snr}(\subref{fig:snr_actual}) displays a $6 \times 6$ occupancy map $x \in [0, 1]^{36}$, where a lighter color indicates a lower occupancy value. We initialize the Seeker’s belief with uniform values $\hat{x}_0 = 0.5 \bm{1}_{36}$ and variance $P_0 = I_{36}$, as shown in Fig.~\ref{fig:snr}(\subref{fig:snr_initial}). We assume that the Supporter observes the entire map, i.e., $C_{B,1} = I_{36}$, and transmits a compressed representation using the compression matrix $\Theta_1=\mathrm{diag}(100 \times I_{18}, I_{18})$ and the Gaussian quantization error $\bm{n}_1 \sim \mathcal{N}(0, I_{36})$.  Due to the diagonal structure of $\Theta_1$, this compression scheme corresponds to $36$ independent scalar measurements, one per cell. The cells in the left section of the map (first 18 entries) have an SNR of   $100$, which is significantly greater than the SNR of $1$ for the right section (last 18 entries). 

We compute $\hat{x}_1$ using the update rule \eqref{eq:est_2} and then project it onto the interval $[0,1]$, resulting in the estimated map shown in Fig.~\ref{fig:snr}(\subref{fig:snr_projeted}).
Comparing this figure with the ground-truth map in Fig.~\ref{fig:snr}(\subref{fig:snr_actual}), we observe that the left side of 
$\hat{x}_1$, corresponding to high-SNR measurements, closely matches the true occupancy values. In contrast, the right side, associated with low-SNR measurements, shows significant deviation.
Additionally, the simulation results reveal that different realizations of $\bm{n}_{1}$ can lead to different values of $\hat{x}_1$ in the right side with considerable variation.
This variability arises because the initial estimation has low precision (high covariance), and the measurement $\bm{y}_{B,t}$ is also of low precision (low SNR) on the right side of the map. As a result, combining them results in an estimate with high covariance $P_1$ and a widely varying $\hat{x}_1$ in that region.

To address this challenge, we adopt a pruning heuristic that excludes components corresponding to small eigenvalues of the SNR matrix. Specifically, we define a modified RSD operator, denoted by $ \textup{MRSD}(A, \tau)$, which behaves like $ \textup{RSD}$ but retains only the components associated with eigenvalues exceeding a predefined threshold $ \tau $ (and not only the ones associated with the strictly-positive eigenvalues). In our framework, we apply $ \textup{MRSD} $ in place of $ \textup{RSD} $ in~\eqref{eq:opt_theta}.  Fig.~\ref{fig:snr}(\subref{fig:snr_modified}) shows the result of deploying MRSD with $\tau=4$, where the effect of discarding low-SNR components on the right side is evident.

\section{Algorithm}
Algorithm~\ref{alg:decoder} outlines the computation of the optimal compression and quantization parameters given the Seeker's current state covariance $P_t^+$ and its planned optimal path $\pi_t^*$.
The function $\textsc{Weight}$ computes the weights (importance) defined in \eqref{eq:w_def}. 
The $\textsc{Permute}$ function constructs permutation matrices $ C_{B,t} $ and $C_{O,t}$ based on the position of the Supporter $\textup{p}_{B,t}$ and its FOV $d_B$. These matrices extract the cells within and outside the Supporter's local map, respectively, and are used to block-partition the state covariance and weighting matrices (Lines~2-3).  The algorithm proceeds by computing a modified weighting matrix $ \tilde{W}_{BB,t}$, using the optimal $N^{*}_{t}$  and  $\Theta^{*}_t$  computed from Theorem~\ref{theo:one}. The operator $\textup{MRSD}$ is used in Line~7 to avoid the low SNR compressions.

%%%%%%%%%%%%%%%%%%%%%%%%%%%%%%%%%%%%%%%%%%%%%%%%%%%%%%

\begin{algorithm}[t]
\caption{Compression Design Algorithm}\label{alg:decoder}
\hspace*{\algorithmicindent} \textbf{Input:}{ $P^+_t$, $\pi_{t}^{*}$, $\textup{p}_{B,t}$, $d_B$, $\tau>0$} 
% \\ \hspace*{\algorithmicindent} \textbf{Output: }{ $\Sigma^N_t$, $\Theta_t$}
\begin{algorithmic}[1]

    \State $W_t\leftarrow$\textsc{Weight}$(\pi_{t}^{*})$
    \State $(P^+_{BB,t}, P^+_{OB,t}, P^+_{OO,t}) \leftarrow$ 
    \textsc{Permute}$(P^+_{t}, \textup{p}_{B,t}, d_B)$
    
    \State $(W_{BB,t}, \sim,  W_{OO,t}) \leftarrow$
    \textsc{Permute}$(W_{t}, \textup{p}_{B,t}, d_B)$ 

    \State Compute $Q_t = P_{OB,t}^+ (P_{BB,t}^+)^{-1}$ and $\tilde{W}_{BB,t} = W_{BB,t} + Q_{t}^\top W_{OO,t} Q_{t}$
    
    \State $\displaystyle (\operatorname*{diag}_{1\leq i \leq d_B}(\sigma_{i,t}^2), U_t) \leftarrow$ \textsc{SVD}$(\tilde{W}_{BB,t}^{1/2} P^+_{BB,t} \tilde{W}_{BB,t}^{1/2})$

    \State $\displaystyle \Sigma_t = \operatorname*{diag}_{1 \leq i \leq d_{B}}(\max\{0, \frac{2}{\alpha}- \frac{1}{\sigma_{i,t}^2}\})$

    \State $(\Theta_t^{*}, (N_{t}^{*})^{-1}) \leftarrow$ \textsc{MRSV}$( \tilde{W}_{BB,t}^{1/2} U_t^\top \Sigma_t  U_t \tilde{W}_{BB,t}^{1/2}, \tau)$
    
    \State \Return $N^{*}_{t}$  and  $\Theta^{*}_t$
    
\end{algorithmic}
\end{algorithm}

%%%%%%%%%%%%%%%%%%%%%%%%%%%%%%%%%%%%%%%%%%%%%%%%%%%%%%%%%%%%%%%%
\begin{figure*} 
\hspace*{1.0em}
\begin{subfigure}[b]{0.32\textwidth} 
    \centering
    \includegraphics[trim = 8cm 3cm 6cm 2cm, clip, width=\textwidth]{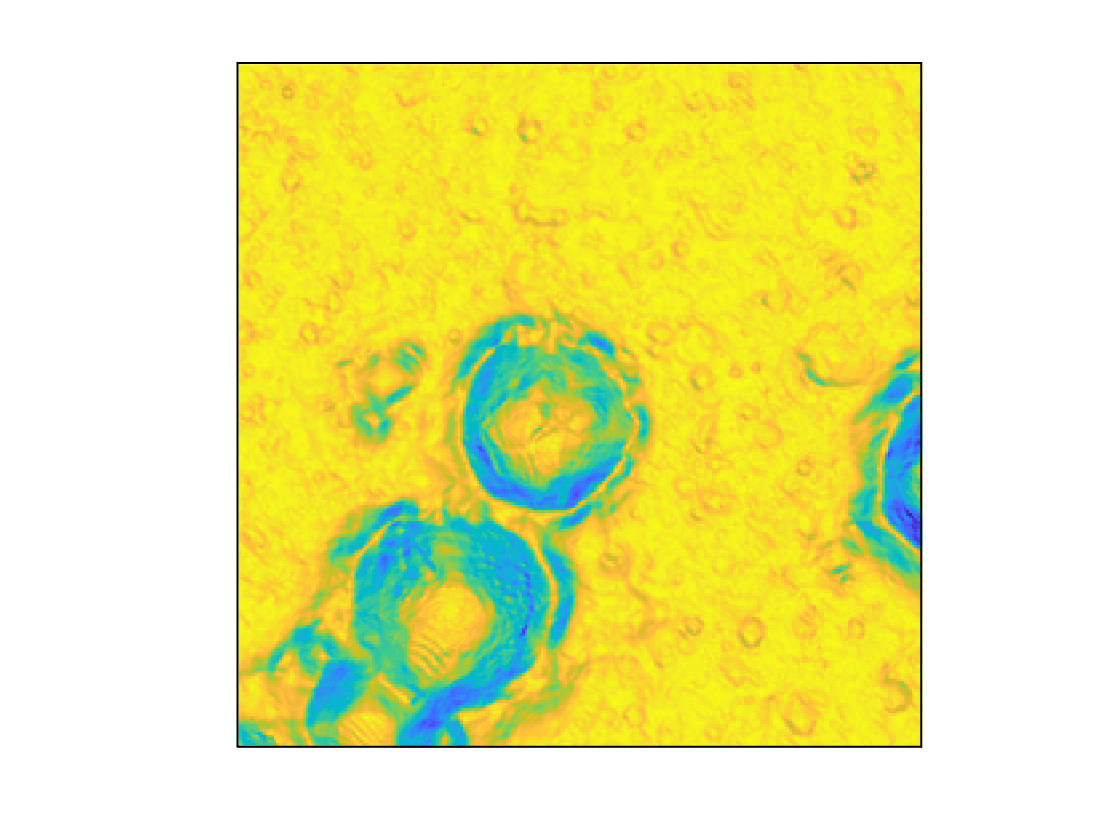}
    \caption{Full resolution traversability map.}
    \label{fig:mars_full}
\end{subfigure}
\begin{subfigure}[b]{0.32\textwidth}
    \centering
\includegraphics[trim = 8cm 3cm 6cm 2cm, clip, width=\textwidth]{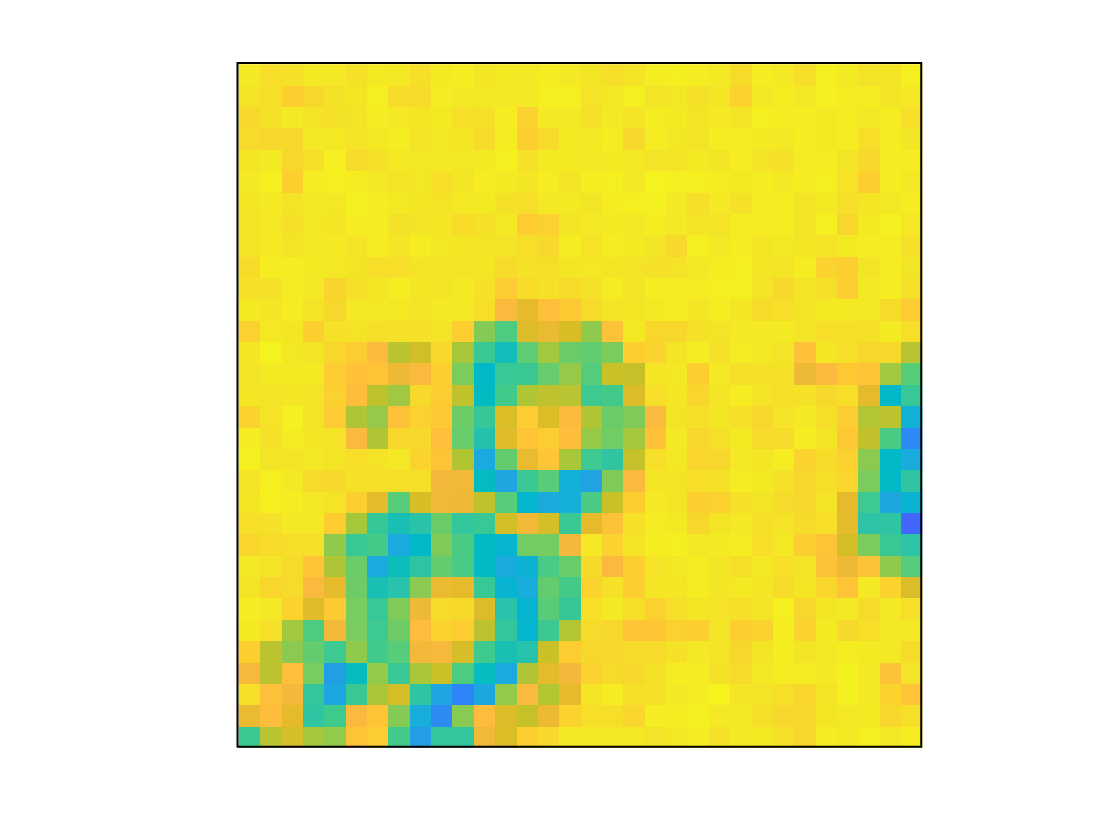}
    \caption{Prior belief $\hat{x}_0$.}
    \label{fig:mars_init}
\end{subfigure}
\begin{subfigure}[b]{0.32\textwidth}
    \centering 
    \includegraphics[trim = 8cm 3cm 6cm 2cm, clip, width=\textwidth]{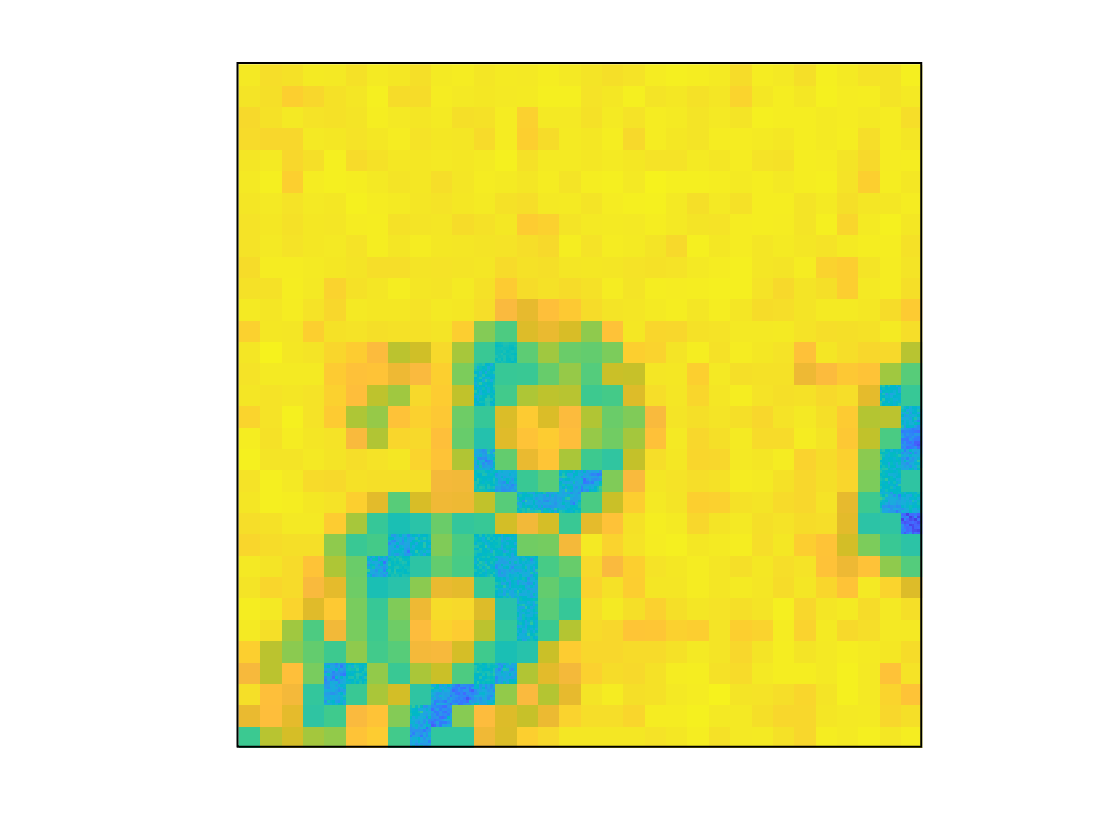}
    \caption{Posterior belief for $\alpha=0.01$.}
    \label{fig:mars_10}
\end{subfigure} 
% \vspace{1em}
\\ 
\hspace*{1.0em}
\begin{subfigure}[b]{0.32\textwidth}
    \centering
\includegraphics[trim = 8cm 3cm 6cm 2cm, clip, width=\textwidth]{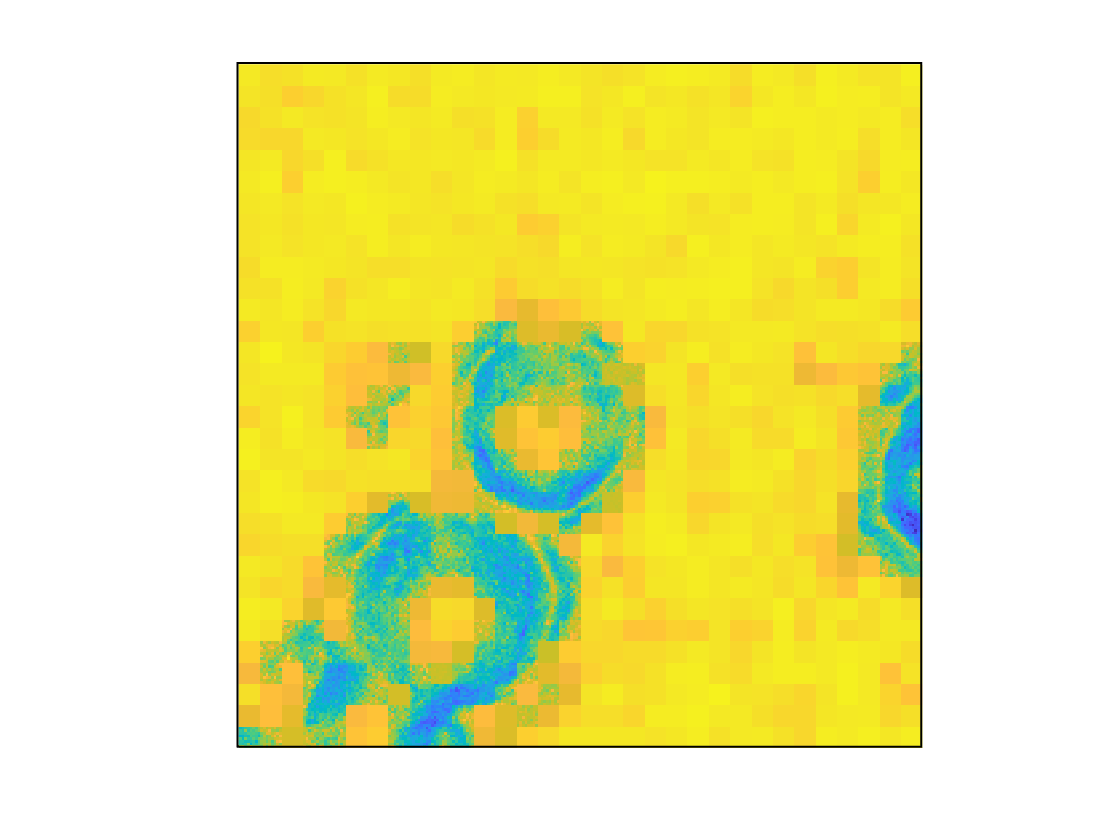}
    \caption{Posterior belief for $\alpha=0.005$.}
    \label{fig:mars_5}
\end{subfigure}
\begin{subfigure}[b]{0.32\textwidth}
    \centering
\includegraphics[trim = 8cm 3cm 6cm 2cm, clip, width=\textwidth]{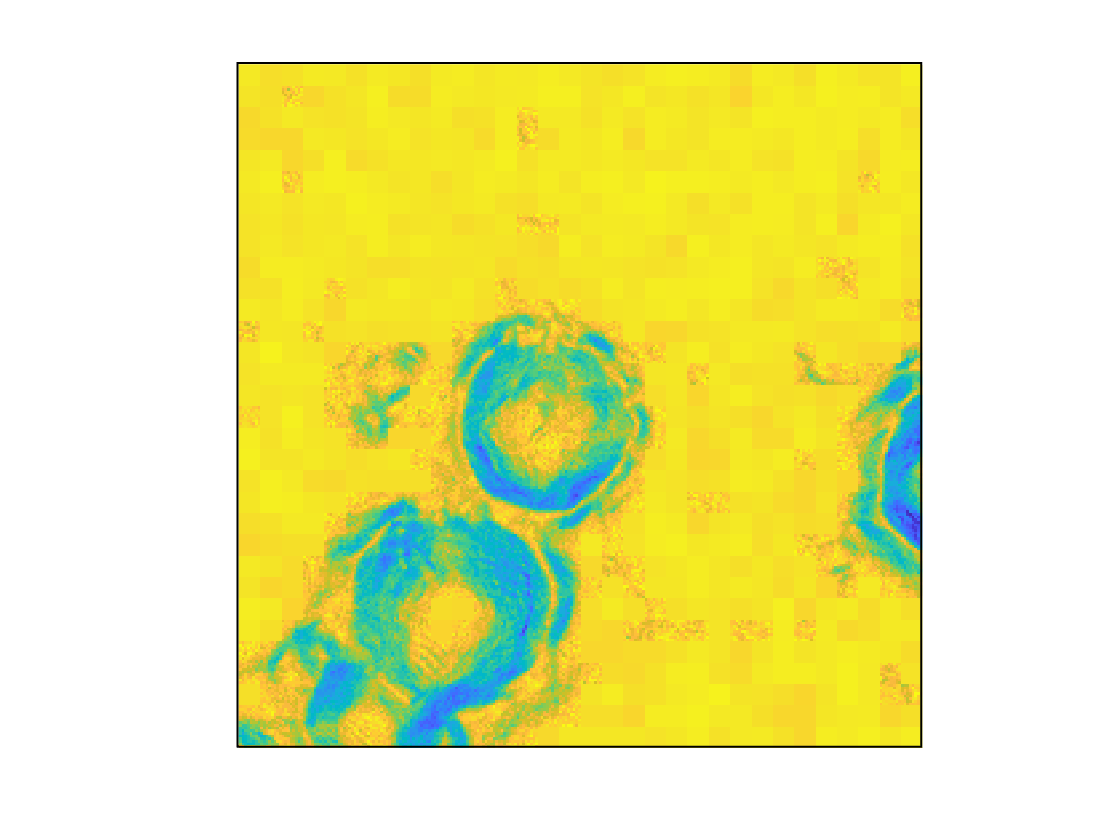}
    \caption{Posterior belief for $\alpha=0.002$.}
    \label{fig:mars_2}
\end{subfigure}
\begin{subfigure}[b]{0.32\textwidth}
    \centering
\includegraphics[trim = 8cm 3cm 6cm 2cm, clip, width=\textwidth]{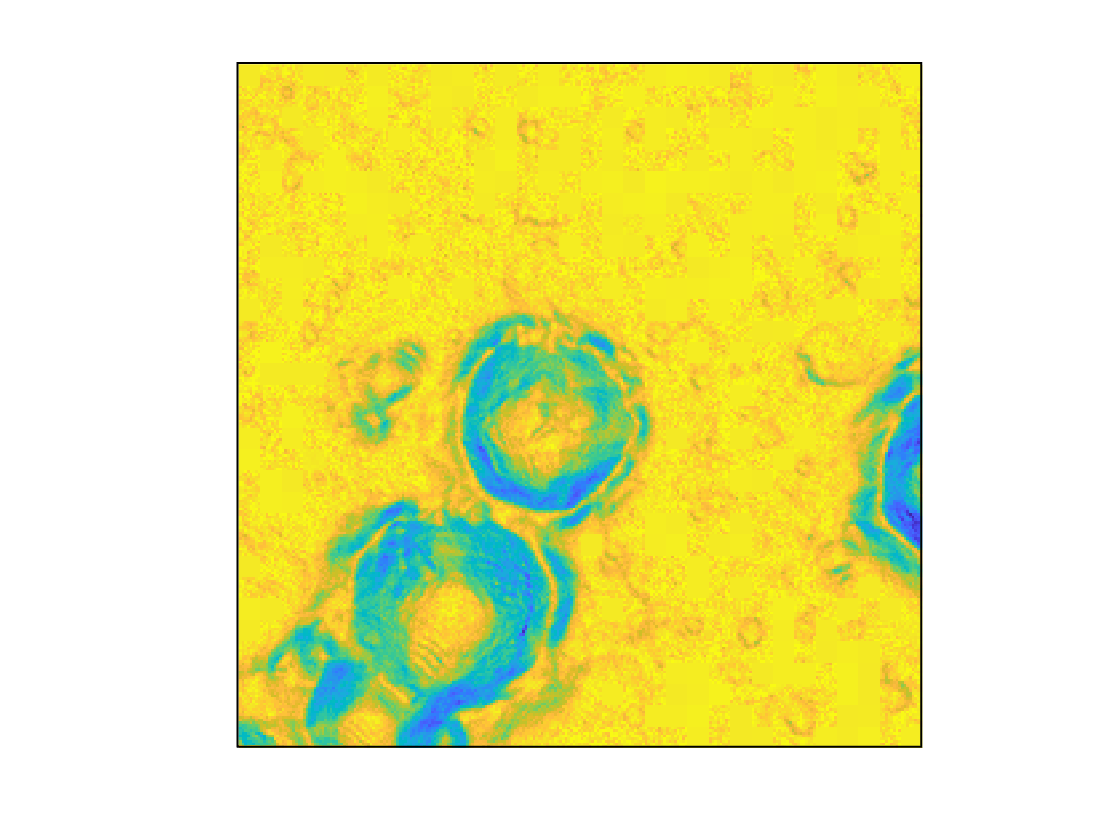}
    \caption{Posterior belief for $\alpha=0.001$.}
    \label{fig:mars_1}
\end{subfigure}
\caption{Effect of bit-rate weight $\alpha$ on optimal compression and the Seeker’s posterior belief.
(\subref{fig:mars_full}) Ground-truth traversability map of the Martian surface, derived from depth data.
(\subref{fig:mars_init}) Initial prior belief $\hat{x}_0$ generated by block-wise averaging.
(\subref{fig:mars_10})–(\subref{fig:mars_1}) Posterior beliefs $\hat{x}_1$ after receiving compressed information from the Supporter for $\alpha = 0.01$, $0.005$, $0.002$, and $0.001$, respectively.}
\label{fig:mars}
\end{figure*}

\begin{figure}[ht]
    \centering        
{\includegraphics[trim = 0cm 0.4cm 0cm 0cm, clip, width=0.9\linewidth]{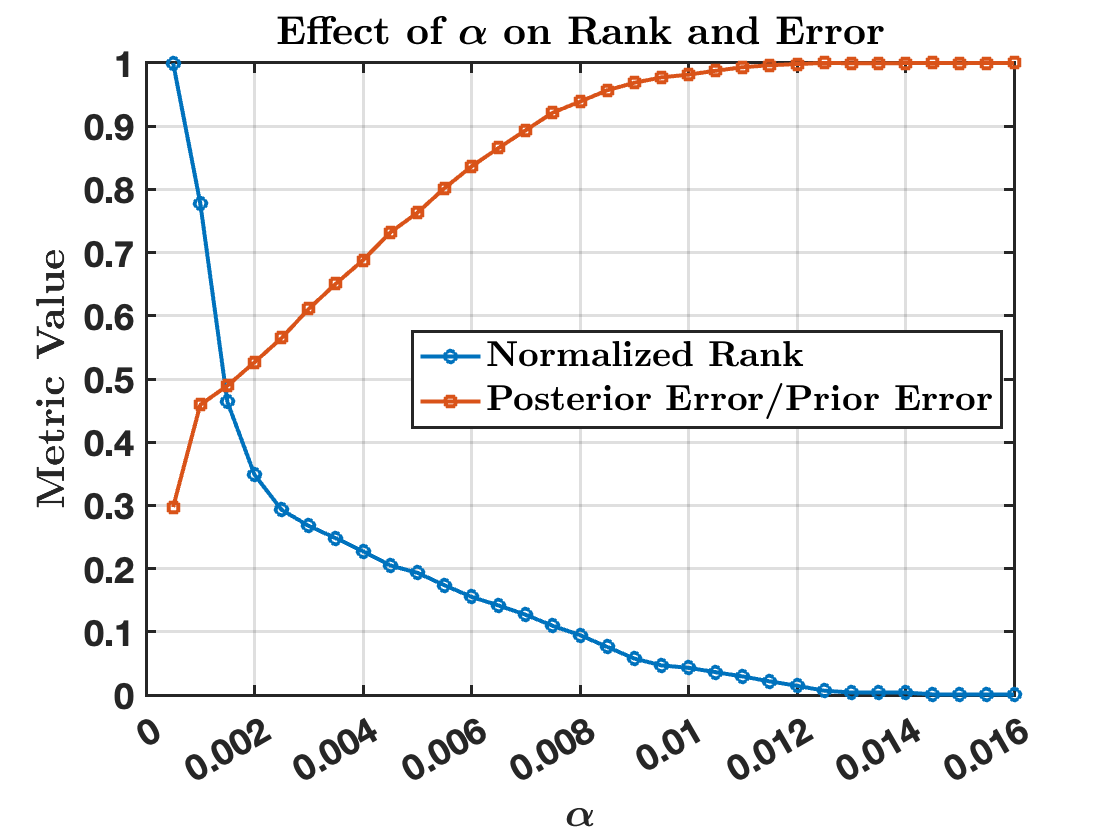} 
    \caption{Effect of bit-rate weight $\alpha$ on compression quality and reconstruction accuracy.
    The blue curve shows the normalized rank of the optimal compression matrix $\Theta^*$. The orange curve represents the ratio between the posterior and prior reconstruction errors, $\|x - \hat{x}_1\| / \|x - \hat{x}_0\|$.}
    \label{fig:alpha_curve}}
\end{figure}

\section{Experiments}
We evaluated our proposed compression framework in two simulation scenarios. In the first simulation, we use a Mars environment depicted in Fig.~\ref{fig:mars}(\subref{fig:mars_full}) and consider a one-shot compression setup. Here, the Supporter has a wide FOV and observes the entire map, and designs a single compressed representation. The objective is to examine how varying the weight on the bit-rate, i.e., the parameter $\alpha$ in Equation~\eqref{eq:alpha_def}, affects the structure and resolution of the optimal compression scheme.  The second scenario takes place in a sample Earth traversability map depicted in Fig.~\ref{fig:map} and features a sequential map compression and communication setup. In this case, the Seeker aims to reach a known destination, while the Supporter follows a predefined path. Both robots continue to observe new parts of the environment as they move. Here, the Supporter dynamically updates and transmits a compressed map over time, selectively retaining and refining task-relevant information.

\subsection{Mars Environment} 
Fig.~\ref{fig:mars}(\subref{fig:mars_full}) shows a $256 \times 256$ elevation map of the Martian surface in the Acidalia Planitia region, obtained from the publicly available HiRISE dataset~\cite{HiRISE2024}, where darker regions indicate areas
more difficult to traverse, and the lighter regions with a yellow
hue indicate areas that are easier to traverse. Let $[y]_j$ denote the raw elevation value at the $j$-th cell in the map. To derive a traversability value, we first compute a local elevation variability measure $[z]_j$, defined as
\[
[z]_j = \textstyle\sum_{j' \in \mathcal{J}} \left| [y]_j - [y]_{j'} \right|,
\]
where $\mathcal{J}$ is the set of neighboring indices around cell $j$. We normalize this value to obtain a proxy for inclination as follows
\[
[x]_j = \frac{[z]_j - \min z}{\max z - \min z}.
\]
The normalized value $[x]_j \in [0,1]$ serves as a proxy for terrain steepness and is used as the traversability value.

Fig.~\ref{fig:mars}(\subref{fig:mars_init}) shows the Seeker's initial belief $\hat{x}_0$, synthetically generated by dividing the traversability map $x$ into $8 \times 8$ blocks and assigning each block the average of its values. The initial uncertainty is modeled by a diagonal covariance matrix $P_0 = 0.001 I$. The weight matrix is defined as $W = \mathrm{diag}([\hat{x}_0]_1 + 0.001, \dots, [\hat{x}_0]_d + 0.001)$, where $d = 256 \times 256$ and the constant $0.001$  is added to ensures numerical stability during the computations in Algorithm~\ref{alg:decoder}, where $\tau=1$ is used in the MRSV function. Intuitively, $W$ assigns higher weights to regions with higher initial occupancy, emphasizing them as more relevant for compression.

We performed simulations for $\alpha \in [0.005, 0.016]$. Figs.~\ref{fig:mars}(\subref{fig:mars_10})–(\subref{fig:mars_1}) show the Seeker's posterior belief after receiving the compressed map for $\alpha = 0.01$, $0.005$, $0.002$, and $0.001$, respectively. The results indicate that for $\alpha \geq 0.016$, the optimal compression is empty, that is, no data is transmitted. In this regime, the communication cost outweighs the potential benefit, and thus, it is better to avoid communication.
As $\alpha$ decreases, transmitting information about the more important regions (the dark areas in our map) becomes increasingly beneficial, while the less critical regions (the yellow areas) are ignored. This behavior is confirmed in Figs.~\ref{fig:mars}(\subref{fig:mars_10})–(\subref{fig:mars_2}), where lowering the value of $\alpha$ monotonically leads to more accurate reconstruction of high-importance areas, while the estimation of less relevant regions remains coarse.
For $\alpha < 0.001$, it becomes cost-effective to include even less important regions in the compression, resulting in a more accurate posterior estimate across the entire map (including yellow region), as shown in Fig.~\ref{fig:mars}(\subref{fig:mars_1}).

We analyze the compression scheme using two metrics shown in Fig.~\ref{fig:alpha_curve}: the normalized rank of the optimal compression matrix $\mathrm{rank}(\Theta^*)/d$ and the $\ell_2$-norm of the error between the actual map $x$ and both the prior and posterior beliefs of the Seeker, defined as $\|x-\hat{x}_1\| / \|x-\hat{x}_0\|$. As expected, the rank of $\Theta^*$ is a monotonically decreasing function of $\alpha$.
Notably, for $\alpha < 0.0005$, $\Theta^*$ becomes full rank. However, due to a non-zero quantization error $N^*$, reconstruction errors persist, which explains why the posterior-to-prior error remains around $0.3$ for $\alpha = 0.0005$. As $\alpha$ decreases further, the quantization becomes increasingly fine, leading to improved reconstruction accuracy and a corresponding reduction in posterior error. While not shown in the figure, in the theoretical limit as $\alpha$ approaches zero, communication becomes effectively free, the quantization noise $N^*$ approaches zero, and the posterior belief converges to the ground truth, yielding zero $\ell_2$ error at the cost of infinite bit-rate.

%%%%%%%%%%%%%%%%%%%%%%%%%%%
\begin{figure*}
\hspace*{1.0em}
\begin{subfigure}[b]{0.32\textwidth}
    \centering
    \includegraphics[trim = 9cm 5cm 8cm 4cm, clip, width=\textwidth]{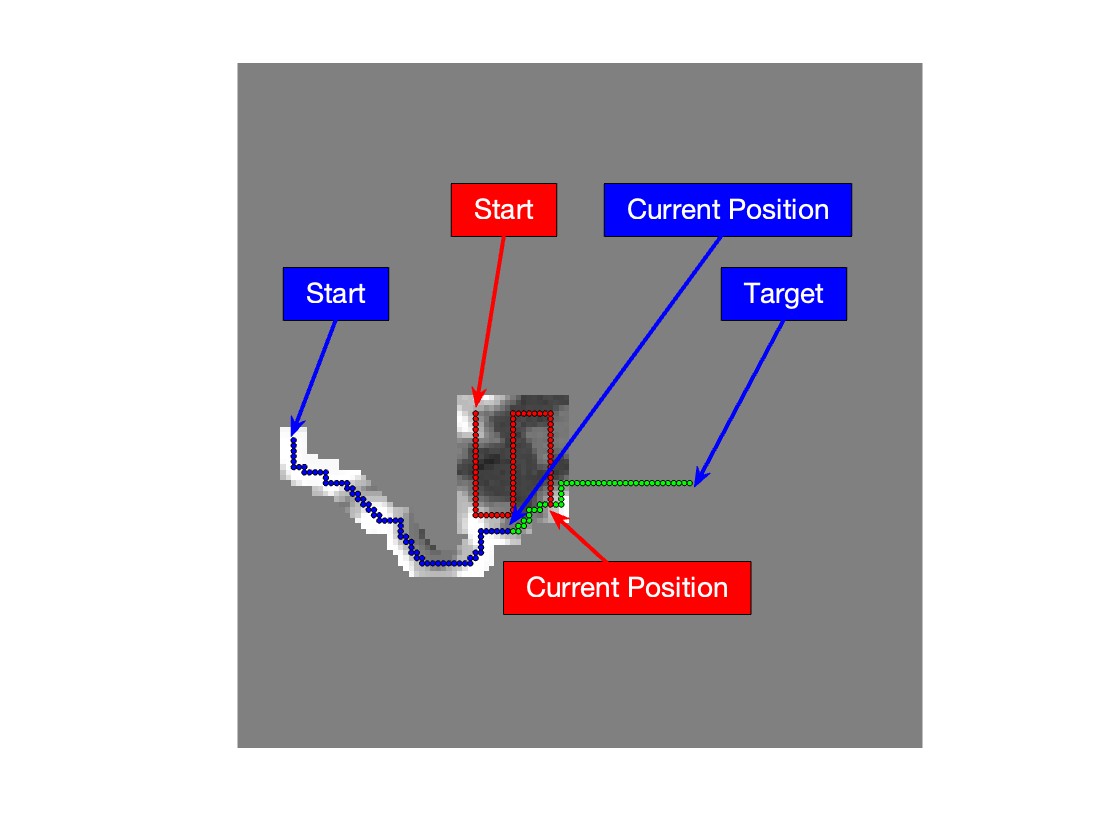}
    \caption{$\alpha=0.0005$ and $t=70$}
    \label{fig:low_alpha_inter}
\end{subfigure}
\begin{subfigure}[b]{0.32\textwidth}
    \centering
\includegraphics[trim = 9cm 5cm 8cm 4cm, clip, width=\textwidth]{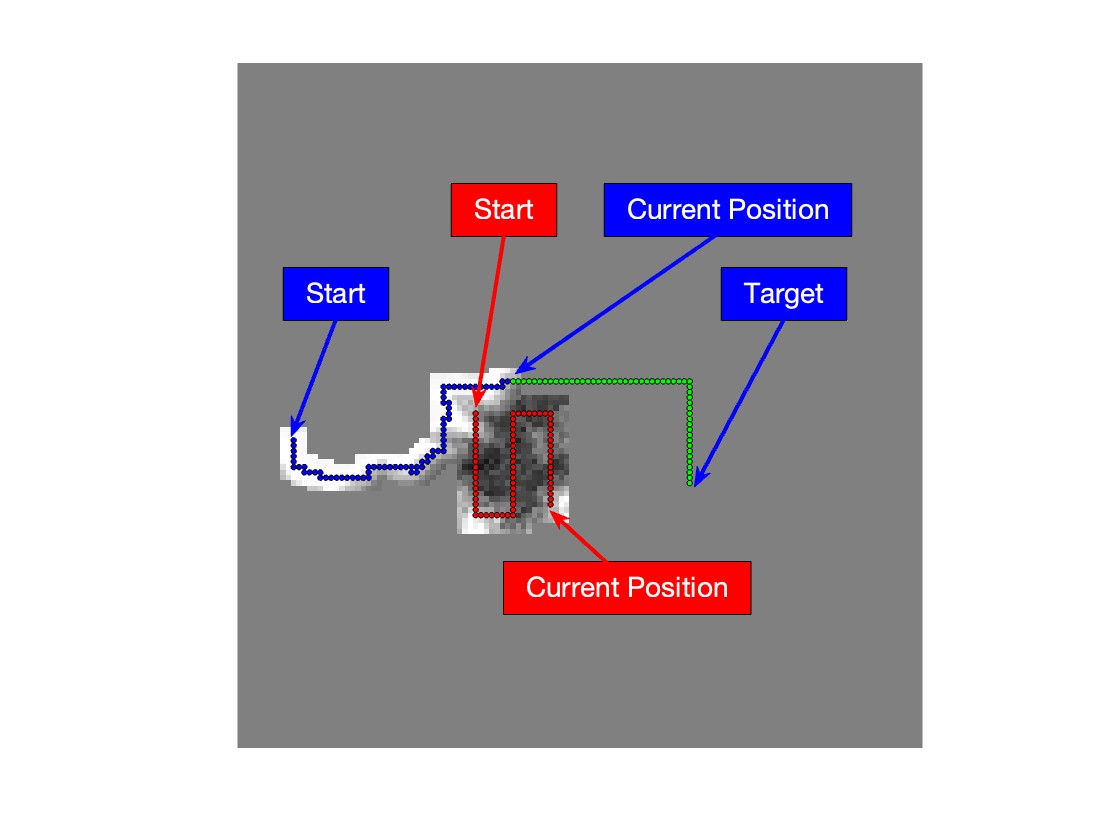}
    \caption{$\alpha=0.05$ and $t=70$.}
    \label{fig:mid_alpha_inter}
\end{subfigure}
\begin{subfigure}[b]{0.32\textwidth}
    \centering \includegraphics[trim = 9cm 5cm 8cm 4cm, clip, width=\textwidth]{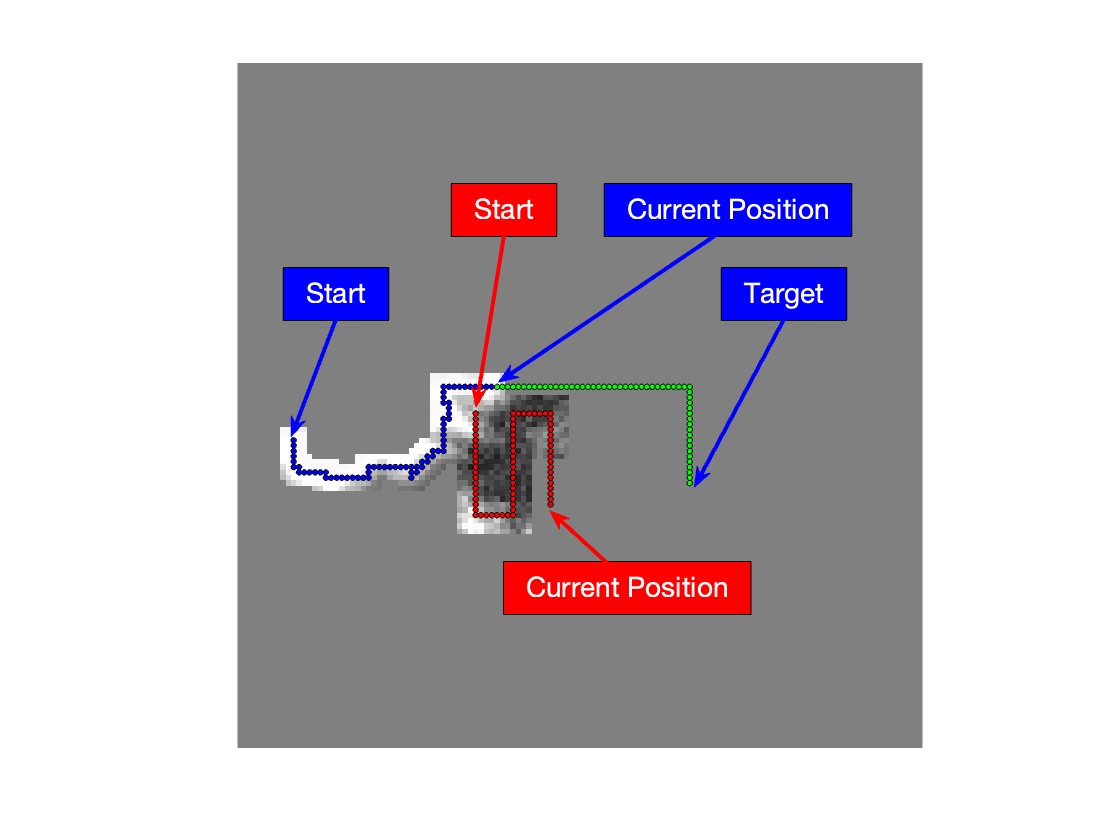}
    \caption{$\alpha=0.9$ and $t=70$.}
    \label{fig:high_alpha_inter}
\end{subfigure} 
% \vspace{0.2cm}
\\
\hspace*{1.0em}
\begin{subfigure}[b]{0.32\textwidth}
    \centering
\includegraphics[trim = 9cm 5cm 8cm 4cm, clip, width=\textwidth]{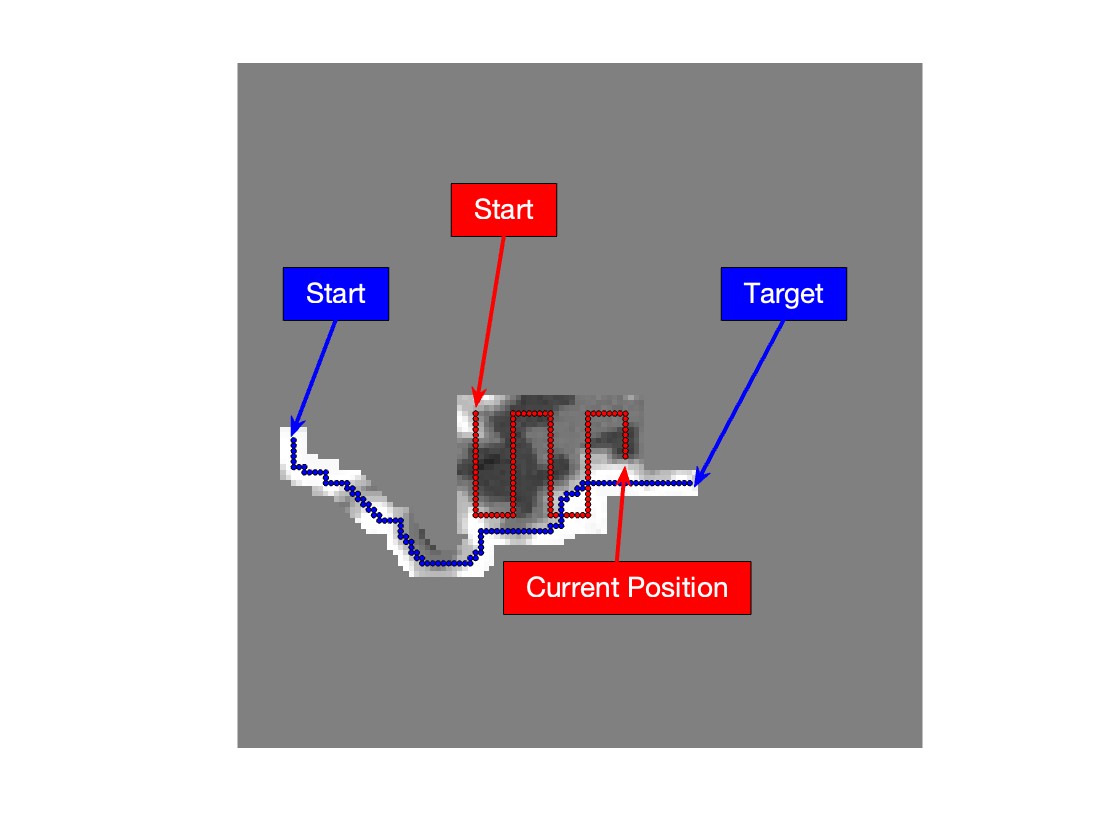}
    \caption{$\alpha=0.0005$ and final time step.}
    \label{fig:low_alpha_final}
\end{subfigure}
\begin{subfigure}[b]{0.32\textwidth}
    \centering
    \includegraphics[trim = 9cm 5cm 8cm 4cm, clip, width=\textwidth]{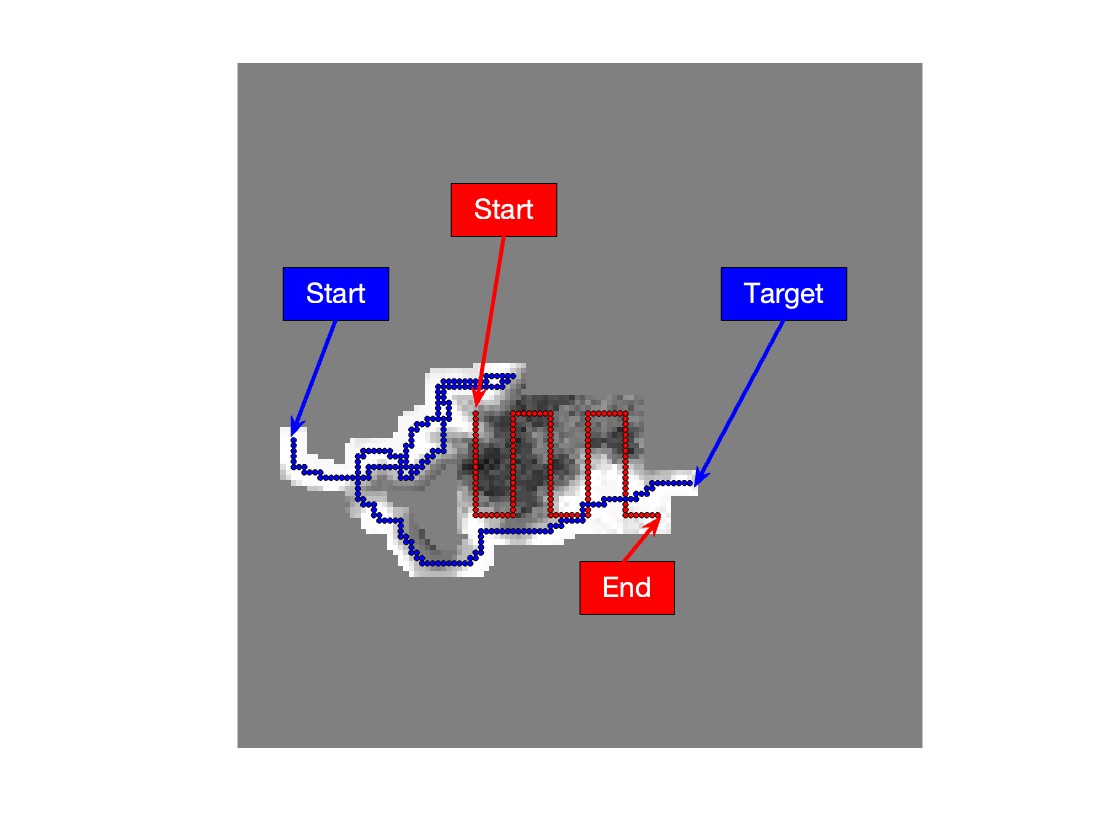}
    \caption{$\alpha=0.05$ and final time step.}
    \label{fig:mid_alpha_final}
\end{subfigure}
\begin{subfigure}[b]{0.32\textwidth}
    \centering
    \includegraphics[trim = 9cm 5cm 8cm 4cm, clip, width=\textwidth]{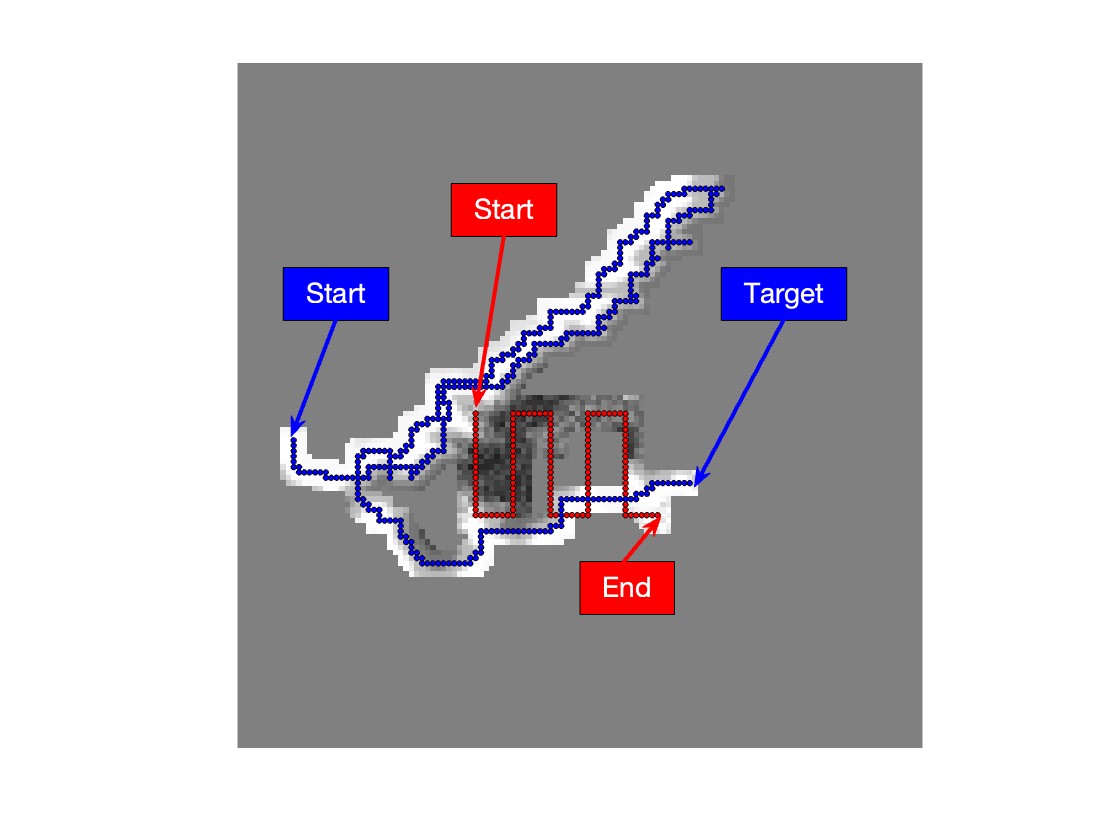}
    \caption{$\alpha=0.9$ and final time step.}
    \label{fig:high_alpha_final}
\end{subfigure}
\caption{Figures (a)-(f) illustrate the Seeker's estimated map for different values of $\alpha$ at the intermediate time step $t=70$, and final time step, when the goal is reached. The Supporter's path is illustrated in red. The cells that the Seeker has already traversed are presented in blue, while green is the current path constructed by its path-planning algorithm.}
\label{fig:sanpshots}
\end{figure*}

% %%%%%%%%%%%%%%%%%%%%%%%%%%%%%%%%%%%%%%%%%%%%%%%%
\begin{figure}[ht]
    \centering        
{\includegraphics[trim = 0cm 0.0cm 0cm 0cm, clip, width=0.9\linewidth]{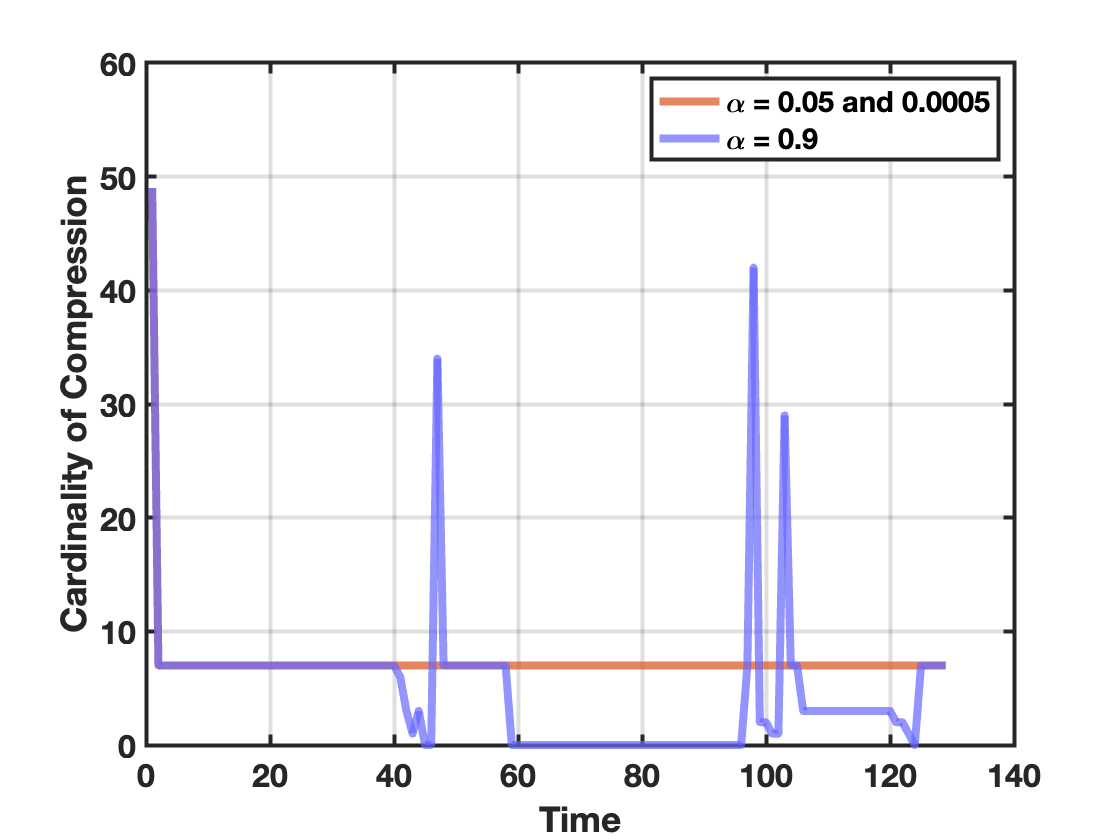} 
    \caption{Cardinality of optimum compression for different values of $\alpha =0.0005, 0.005$ and $0.9$.}
    \label{fig:abst_profile}}
\end{figure}
\begin{figure}[ht]
    \centering        
{\includegraphics[trim = 0cm 0.0cm 0cm 0cm, clip, width=0.9\linewidth]{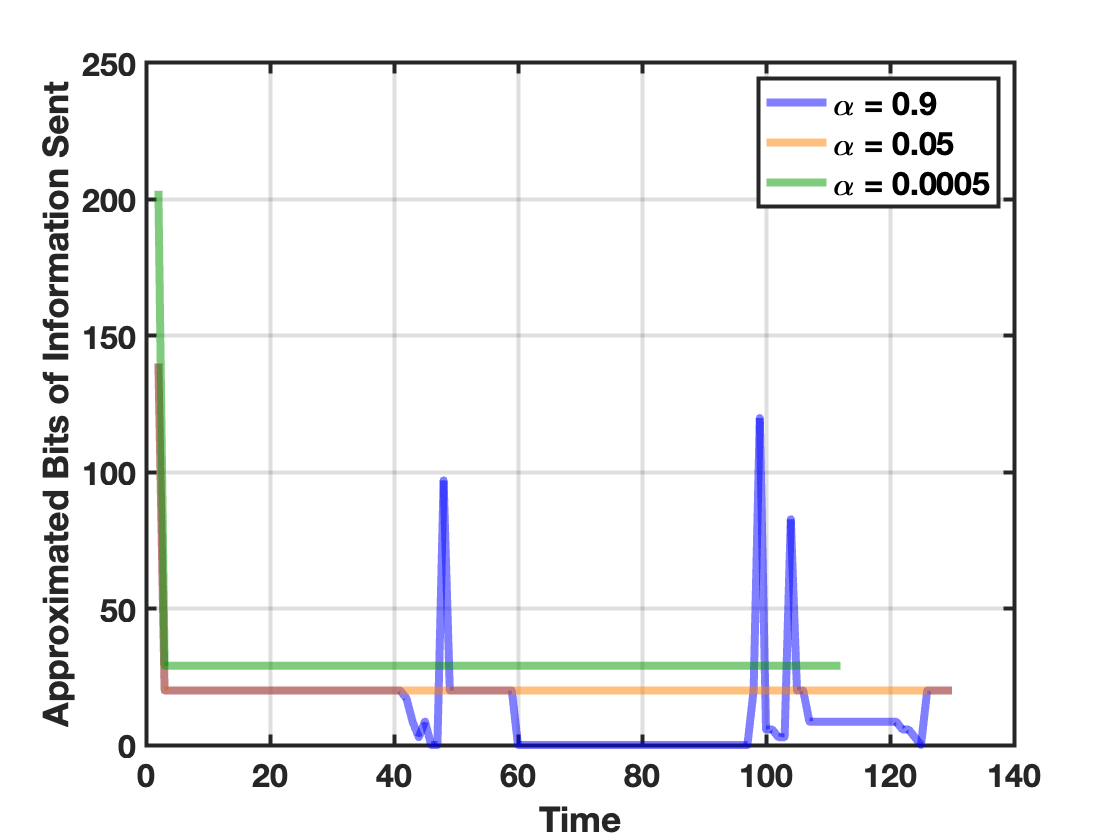} 
    \caption{Approximated bits of information the Supporter sent to the Seeker at each time step for different values of $\alpha =0.0005, 0.005$ and $0.9$.}
    \label{fig:bits}}
\end{figure}

\subsection{Earth Environment}
For this setting, we conducted simulations on a 2D Earth map depicted in Fig.~\ref{fig:map}, represented as an $128 \times 128$ grid world.
The Supporter can move over obstacles (i.e., it is an aerial vehicle, such as a surveillance drone). 
The Seeker's FOV is $5\times 5$ cells, while the Supporter has a FOV of 7 × 7 cells. Both robots are positioned at the center of their respective local maps. We tested our framework for different values of $\alpha$, while maintaining the same initial and destination positions of the Seeker indicated by the blue ``Start" and ``Target" labels in sub-figures of Fig.~\ref{fig:sanpshots}. The red path represents the Seeker’s trajectory, beginning at the red ``Start" position and ending at the red ``End" point.
The Seeker's path-planner runs $A^*$ algorithm and uses the cell cost given in \eqref{eq:cost_def} with values of the constants $a = 0.025$ and $\epsilon= 0.501$. The cells that the Seeker has already traversed are presented in blue, while green is the current path constructed by its path-planning algorithm. The Seeker starts from the initial belief $\bm{x}_0\sim \mathcal{N}(0.5 \bm{1}_d, I_d)$, and the weights $w_t$ are defined with the parameter $\sigma = 10$ in \eqref{eq:w_def}.

We performed simulations for three values, namely, $\alpha=0.9$, $\alpha=0.05$, and $\alpha=0.0005$, by running Algorithm~\ref{alg:decoder} within the Compression Scheme Module (see Fig.~\ref{fig:framework}), with $\tau=1.4$. We also simulated the fully-informed (FI) strategy, where the Supporter perfectly sends occupancy values of cells in its FOV(corresponding to $\Theta_t=I_{d_B}$ and $N_t=0$ for all $t$), and the uninformed (U) strategy, where the Supporter sends no information at all (corresponding to $N_t=\infty$ for all $t$).
Fig.~\ref{fig:sanpshots} illustrates the Seeker's estimated map in gray scale for these values of $\alpha$ at the intermediate time step $t=70$, and the final time step, that is, when the goal is reached. As seen in Figs.~\ref{fig:sanpshots}(\subref{fig:low_alpha_final})-(\subref{fig:high_alpha_final}), higher values of $\alpha$ lead to more aggressive compression and quantization, resulting in lower-quality map estimates on the Seeker’s side and consequently longer paths to the goal.

The cardinality of the selected compression, that is, $\mathrm{rank}(\Theta_t^*)$ at each time step is shown in Fig.~\ref{fig:abst_profile}. This value for $\alpha=0.05$ and $\alpha= 0.0005$ is $49$ at the first time step and $7$ afterwards. 
This result suggests that, for small values of $\alpha$, the optimal policy for the Supporter is to initially transmit a full-resolution (albeit noisy) representation of its local map. In subsequent steps, it sends only the information corresponding to the seven newly observed cells. For  $\alpha=0.9$, the cardinality of optimum compression varies drastically over time. For example, between time steps $60$ and $95$, it is strategic to send no information, as the Supporter's local map is not informative. Also, at certain instances it is beneficial to send maps with cardinality of more than $7$ (the number of newly observed cells).   
Fig.~\ref{fig:bits} 
shows the approximated bits of information the Supporter sent to the Seeker over time, computed as 
\[
b_t \triangleq I(\bm{o}^G_{B,t} ; \bm{y}^G_{B,t}), 
\]
as computed in Theorem~\ref{theo:one}.
This figure shows that although the cardinality of the selected compression for $\alpha=0.05$ and $\alpha= 0.0005$ are equal, the selected quantization levels differ. The quantization for $\alpha= 0.0005$ is more precise and results in a shorter path length for the Seeker, as seen in Fig.~\ref{fig:sanpshots}(\subref{fig:low_alpha_final}) compared to Fig.~\ref{fig:sanpshots}(\subref{fig:mid_alpha_final}). Conversely, this finer quantization  requires a higher bit rate to be transferred.

%%%%%%%%%%%%%%%%%%%%%%%%%%%
\begin{table}[t]
\caption{Simulation results.}
\label{tab:results}
\centering
\small
\begin{tabular}{ 
  |>{\centering\arraybackslash}p{0.6cm}
  |>{\centering\arraybackslash}p{0.6cm}
  |>{\centering\arraybackslash}p{1.6cm}
  |>{\centering\arraybackslash}p{1.3cm}
  |>{\centering\arraybackslash}p{1.1cm}
  |>{\centering\arraybackslash}p{0.6cm}| } 
% \begin{tabular}{ |p{0.6cm}|p{0.6cm}|p{1.6cm}|p{1.3cm}|p{1.1cm}|p{0.6cm}| } 
    \hline
    & {FI} & {$\alpha=0.0005$} & {$\alpha=0.05$} &  {$\alpha= 0.9$} &  {U}  \\
    \hline
    \(r_\textrm{avg}\) & 49 &  7.33 & 7.33  &  4.87 & 0  \\ 
    \hline
    \(t_\textrm{reach}\) & 112 & 112 & 224  & 390 & 424  \\ 
    \hline
    \(c_\textrm{reach}\) & 14.81 & 14.81 & 33.79  & 72.56 & 95.47  \\ 
    \hline
    \(b_\textrm{avg}\) & $\infty$ & 61.17 & 41.80  & 27.83 & 0  \\ 
    \hline
\end{tabular}
\end{table}

%%%%%%%%%%%%%%%%%%%%%%%%%%%%%%%%%%%%%%%%%%%%%%%%%%%%%
Table~\ref{tab:results} compares the average cardinality of the optimal compression $r_\textrm{avg}$ and the average bit rate $b_\textrm{avg}$, both computed over the time horizon required to reach the destination, as well as the time to reach the goal cell  $t_\textrm{reach}$ and the total cost of the traversed path $c_\textrm{reach}$ (defined as the sum of traversability values along the path). As expected, both the time to reach the goal and the path cost increase with higher values of $\alpha$. In contrast, the compression cardinality and the communication rate decrease.

\section{Conclusions}
This work presents a principled framework for communication-aware map compression in collaborative navigation, enabling a Supporter agent, such as an aerial vehicle, to transmit compressed, task-relevant information to a Seeker agent, such as a ground vehicle. The map is represented as a set of traversability values assigned to individual cells in the environment. The proposed framework constructs a linear compression by selecting weighted groupings of these cells and by assigning each group a resolution level, which directly determines its bit-rate. This flexible design allows the robot to allocate higher bit-rates to critical regions, and assigns lower, or even zero, bit-rate to uninformative areas.

A key contribution of this work is the derivation of a closed-form solution to the compression design problem, utilizing a reverse water-filling approach from rate-distortion theory, which enables efficient, real-time computation. The proposed framework enables the Seeker to infer the Supporter’s compression decisions from shared trajectory data, removing the need to transmit the compression structure, thereby reducing communication overhead.
Future work will explore extending this framework to richer map representations such as mixed geometric-semantic maps, topological abstractions, or scene graphs, which can encode object-level and relational information. 

%%%%%%%%%%%%%%%%%%%%%%%%%%%%%%%%%%%%%%%%%%%%%%%%%%%%%%%%%%%%%%%%%%%%%%%%%%%%%%%%
\appendices

\section{Proof of Proposition~\ref{prop}}
\label{appendix_zero}
For $I(\bm{o}_{B,t} ; \bm{y}_{B,t})$, it follows that
\begin{align}
\nonumber
    &I(\bm{o}_{B,t} ; \bm{y}_{B,t}) 
   = h(\bm{y}_{B,t}) - h(\bm{y}_{B,t} | \bm{o}_{B,t}) \\
   \nonumber
   &= h(\bm{y}_{B,t}) - h(\bm{o}_{B,t} + \bm{n}_{t} | \bm{o}_{B,t}) = h(\bm{y}_{B,t})-  h(\bm{n}_{t}),
\end{align}
where the last equality follows from the fact that $\bm{n}_{t}$ and $\bm{o}_{B,t}$ are independent. A similar result can be established for $I(\bm{o}^G_{B,t} ; \bm{y}^G_{B,t})$. Thus, it follows that
\begin{align}
\nonumber
\hspace{-0.6em}
    I(\bm{o}_{B,t} ; \bm{y}_{B,t}) 
   &\!-\! I(\bm{o}^G_{B,t} ; \bm{y}^G_{B,t})\\ \label{eq:I_diff}
   &\!=\!(h(\bm{n}^G_t)\!-\! h(\bm{n}_t)) \!-\! (h(\bm{y}^G_t)\!-\! h(\bm{y}_t)).
\end{align}
On the other hand, for a Gaussian random variable  $\bm{x}^G$ and any other random variable $\bm{x}$ with
the same mean and covariance matrix, we have that 
$\mathrm{KL}(\bm{x}_{t} \| \bm{x}^G_{t})= h(\bm{x}^G_t)- h(\bm{x}_t)$ \cite[Fact~C.4]{silva2009unified}. Using this relation to simplify \eqref{eq:I_diff} establishes \eqref{eq:KL_diff}. Furthermore,
\begin{align*}
    &\mathrm{KL}(\bm{n}_{t} \| \bm{n}^G_{t}) \!=\! h(\bm{n}^G_t)\!-\! h(\bm{n}_t) \!=\! \textstyle\sum_{i=1}^ {d_{\Theta_t}} h([\bm{n}^G_t]_i)\!-\! h([\bm{n}_t]_i)\\
    &=\textstyle\sum_{i=1}^ {d_{\Theta_t}} 
     \frac{1}{2} \log(\frac{2\pi e \ \Delta_{i,t}^2}{12}) -  \log(\Delta_{i,t})\!\!=\! \frac{d_{\Theta_t}}{2} \log \frac{2\pi e}{12},
\end{align*}
where we used the fact that elements of $\bm{n}_t$ (similarly, $\bm{n}^G_t$) are mutually independent.

\section{Proof of Theorem~\ref{theo:one}}
\label{appendix_one}
    \subsection*{Proof of claim~i)}
    Let us define the joint random variable $\bm{z}_t \triangleq [\hat{\bm{x}}_{B,t}^{+\top},
    \hat{\bm{x}}_{O,t}^{+\top}, \bm{y}_{B,t}^\top]^\top\!\!$. It is easy to verify that $\bm{z}_t$ has mean  $\mathbb{E}[\bm{z}_t] = [\hat{x}_{B,t}^{+\top}, \hat{x}_{O,t}^{+\top}, (\Theta_{t} \hat{x}^+_{B,t})^\top]^\top$ and covariance 
    \begin{align*}
        &\mathbb{V}[\bm{z}_t] =\begin{bmatrix}
        P_{BB,t}^+ & P_{BO,t}^+ & P_{BB,t}^+ \Theta_{t}^\top\\
        P_{OB,t}^+ & P_{OO,t}^+ & P_{OB,t}^+ \Theta_{t}^\top\\
        \Theta_{t} P_{BB,t}^+ & \Theta_{t}  P_{BO,t}^+  & \Theta_{t} P_{BB,t}^+ \Theta_{t}^\top + N_{t}
        \end{bmatrix}
        .
    \end{align*}
The LMVE is mathematically equivalent to the marginalization of the jointly Gaussian distribution. Therefore, one may compute the conditional covariance $P_{BB,t+1} = \mathbb{V}[\hat{\bm{x}}^+_{B,t}| \bm{y}_{B,t}]$ as
\begin{align}
    \nonumber
        P_{BB,t+1}&\!=\!P_{BB,t}^{+}\!\!-\!\!P_{BB,t}^{+}
        \Theta_{t}^\top\!(\Theta_{t} P_{BB,t}^+ \Theta_{t}^ \top\!\!+ \!N_{t})^{-1} 
        \Theta_{t} P_{BB,t}^+\\ \label{eq:bb_iter}
        &\!=\!((P_{BB,t}^{+})^{-1}\!+ \Theta_t^\top N_{t} ^{-1}\Theta_t)^{-1},
\end{align}
where in the last inequality we have used the Woodbury matrix inversion Lemma \cite{woodbury1950inverting}.
Similarly, we compute $P_{OO,t+1} = \mathbb{V}[\hat{\bm{x}}^+_{O,t}| \bm{y}_{B,t}]$ as
\begin{align} \nonumber
        &P_{OO,t+1} = \\ \label{eq:dd_iter}
        & P_{OO,t}^{+}\!\!-\!\!P_{OB,t}^{+}
        \Theta_{t}^\top\!
        (\Theta_{t} P_{BB,t} \Theta_{t}^ \top\!\!+\!N_{t}\!)^{-1}  \Theta_{t} P_{BO,t}^+. 
\end{align}
From \eqref{eq:bb_iter}, we have $\Theta_{t}^\top\!
        (\Theta_{t} P_{BB,t} \Theta_{t}^ \top\!\!+\! N_{t})^{-1}  \Theta_{t} = (P_{BB,t}^+)^{-1} (P_{BB,t}^+- P_{BB,t+1}
        )(P_{BB,t}^+)^{-1}$. 
        Substituting this expression in \eqref{eq:dd_iter} yields
\begin{align*}
    P_{OO,t+1} = S_{t}+ Q_{t} P_{BB,t+1} Q_{t}^\top,
\end{align*}
where $Q_{t}\!\triangleq\! P_{OB,t}^+ (P_{BB,t}^+)^{-1}\!$, and $S_{t}\! \triangleq \!P_{OO,t}^{+}\!-Q_{t} P_{BB,t}^{+} Q_{t}$.

\subsection*{Proof of claim~ii)}
The first equality in \eqref{eq:trace} is established as follows:
\begin{align*}
    &\mathrm{tr}(W_{R,t} P_{R,t+1})= \mathrm{tr}(C_{R,t} W_t C_{R,t}^\top C_{R,t} P_{t+1}C_{R,t}^\top)\\
    & =\mathrm{tr}(C_{R,t} W_t P_{t+1}C_{R,t}^\top) = \mathrm{tr}(W_t P_{t+1}C_{R,t}^\top C_{R,t})\\
    &=\mathrm{tr}(W_t P_{t+1}),
\end{align*}
where we have used the fact that $C_{R,t}$ is a unitary matrix and that the trace of a product of matrices is invariant under cyclic permutations. Using the fact that $W_{R,t}$ is diagonal, we have
\begin{align} 
    \label{eq:trace_1}
    &W_{R,t} P_{R,t+1}= \begin{bmatrix}
        W_{BB} P_{BB} & W_{BB} P_{BO}\\
        W_{OO} P_{OB}   &  W_{OO} P_{OO}
    \end{bmatrix}.
\end{align}
Applying the trace operator on both sides of \eqref{eq:trace_1} yields
\begin{align*}
    &\mathrm{tr}(W_{R,t} P_{R,t+1}) =\\
    &\mathrm{tr}(W_{BB,t} P_{BB,t+1}) + \mathrm{tr}(W_{OO,t} P_{OO,t+1})=\\ 
    &\mathrm{tr}(W_{BB,t} P_{BB,t+1}) + \mathrm{tr}(W_{OO,t} S_{t}) 
    \\ & \qquad \qquad \qquad \qquad \  +\mathrm{tr}(W_{OO,t} Q_{t} P_{BB,t+1} Q_{t}^\top),
\end{align*}
where we plugged the value $P_{OO,t}$ from \eqref{eq:p_dd}.
Finally, defining $\tilde{W}_{BB,t} \triangleq W_{BB,t} + Q_{t}^\top W_{OO,t} Q_{t}$ completes the proof.

\subsection*{Proof of claim iii)}
Substituting $P_{BB,t}=C_{B,t} P_t^+ C_{B,t}^\top$ in \eqref{eq:mutal_first}, we have
$ I(\bm{o}^G_{B,t};\bm{y}^G_{B,t})\!=\!
    \frac{1}{2} \log\!\det(\Theta_t P_{BB,t}^+ \Theta_t^\top+ N_{t})-\frac{1}{2} \log\!\det(N_{t})$. 
    Using the matrix determinant Lemma \cite{harville1998matrix}, we have that $\det(\Theta_t P_{BB,t}^+ \Theta_t^\top+N_{t})\!=\! \det(P_{BB,t}^+) \det(N_{t})\\ \det((P_{BB,t}^+)^{-1}+ \Theta_t^\top N_t^{-1} \Theta_t)$. Thus,
    \begin{align*}
I(\bm{o}^G_{B,t};\bm{y}^G_{B,t})\!=& \frac{1}{2} \log\!\det(P_{BB,t}^+)\\ 
&+\frac{1}{2}
\log\!\det((P_{BB,t}^+)^{-1}+ \Theta_t^\top N_{t}^{-1} \Theta_t) \\
=& \frac{1}{2}\log\!\det(P_{BB,t}^+)- \frac{1}{2}\log\!\det(P_{BB,t+1}),
\end{align*}
where we used \eqref{eq:p_bb} in the last equality.

%%%%%%%%%%%%%%%%%%%%%%%%%%%%%%%%%%%%%%%%%%%%%%%%%%%%

\section{Proof of Theorem~\ref{theo:two}}
\label{appendix_two}
Let $\tilde{W}_{BB,t}^{1/2} P^+_{BB,t} \tilde{W}_{BB,t}^{1/2}$ have the eigen-decomposition as $U_t^\top D U_t$, where $D_t= \mathrm{diag}_{1\leq i \leq d_i}( \sigma_{i,t}^2)$ and $U_t$ is a unitary matrix $U_tU_t^\top=U_t^\top U_t=I$. 
By defining the new variable $Q\triangleq U_t \tilde{W}_{BB,t}^{1/2} P_{BB,t+1} \tilde{W}_{BB,t}^{1/2} U_t^\top$, we may verify that $\mathrm{tr}(Q) = \mathrm{tr}(\tilde{W}_{BB,t} P_{BB,t+1})$ and also $\log\!\det(Q) =\log\!\det(P_{BB,t+1}) +\log\!\det(\tilde{W}_{BB,t})$. For $\tilde{W}_{BB,t}^{1/2} \succ 0$, we have that $P_{BB,t+1} \preceq P^+_{BB,t}$ if and only if
\begin{align*}
\tilde{W}_{BB,t}^{1/2} P_{BB,t+1} \tilde{W}_{BB,t}^{1/2} \preceq \tilde{W}_{BB,t}^{1/2} P^+_{BB,t} \tilde{W}_{BB,t}^{1/2},
\end{align*}
which, after conjugation by the unitary matrix $U_t$, is equivalent to $Q \preceq D$. 

By ignoring the constant term $ \log\!\det(\tilde{W}_{BB,t})$, Problem \eqref{eq:opt_convex} can be cast as the following problem in terms of $Q$.
\begin{subequations}
\label{eq:hadamard}
    \begin{align}
    \min_{Q} \quad&  \mathrm{tr}(Q)- \frac{\alpha}{2} \log\!\det(Q),\\
    \text{s.t.,} \quad 
    & Q \preceq D = \mathrm{diag}_{1\leq i \leq d_B}(\sigma_{i,t}^2) .
\end{align}
\end{subequations}
For any positive definite matrix $Q$, Hadamard's inequality states that $\det (Q) \leq \Pi_i Q_{ii}$  and the equality holds if and only if the matrix is diagonal. 
Hence, if the diagonal elements of $Q$ are fixed, $\det(Q)$ is maximized by setting all off-diagonal entries to zero. 
Thus, the optimal solution to \eqref{eq:hadamard} is diagonal.

Writing $Q = \mathrm{diag}_{1\leq i \leq d_B}(q_i)$, the problem is decomposed as $d_B$ independent optimization problems, each of which minimizes $q_i-(\alpha/2) \log q_i$ subject to $0\leq q_i\leq \sigma_{i,t}^2$. 
Direct calculation yields that the optimal solution is $q^*_i=\min(\alpha/2, \sigma_{i,t}^2)$, and thus \[Q^*=\operatorname*{diag}_{1\leq i \leq d_B} (\min\{ \frac{\alpha}{2}, \sigma_{i,t}^2\}).
\]
To compute $\Theta_t^*$, we first need to compute $(P_{BB,t+1}^*)^{-1}$. From the definition of $Q$, we have 
\begin{align*}
(Q^*)^{-1} &= \operatorname*{diag}_{1\leq i \leq d_B} (\max \{ \frac{2}{\alpha}, \frac{1}{\sigma_{i,t}^2}\}) \\
&= U_t \tilde{W}_{BB,t}^{-1/2} (P_{BB,t+1})^{-1} \tilde{W}_{BB,t}^{-1/2} U_t^\top.
\end{align*}
Therefore, 
\begin{align*}
(P_{BB,t+1}^*)^{-1}\!=\!\tilde{W}_{BB,t}^{1/2} U_t^\top & \!\!\!\operatorname*{diag}_{1\leq i \leq d_B}(\max \{\frac{2}{\alpha}, \frac{1}{\sigma_{i,t}^2}\}) U_t \tilde{W}_{BB,t}^{1/2}.
\end{align*}
From the definition of $D_t$, we have \[(P_{BB,t}^+)^{-1} = \tilde{W}_{BB,t}^{1/2} U_t^\top \!\!\!\operatorname*{diag}_{1\leq i \leq d_B}(\frac{1}{\sigma_{i,t}^2})  U_t \tilde{W}_{BB,t}^{1/2}.
\]
 Substituting these values in \eqref{eq:theta_calc}, the optimal compression $\Theta_t^*$ and quantization $N_{t}^{*}$ satisfy
\begin{align*}
   &{\Theta_t^*}^\top (N_{t}^{*})^{-1} \Theta_t^* =  {(P_{BB,t+1}^*)}^{-1} - (P^+_{BB,t})^{-1} \\
   &= \tilde{W}_{BB,t}^{1/2} U_t^\top \! \operatorname*{diag}_{1\leq i \leq d_B}(\max\{0, \frac{2}{\alpha}\!-\! \frac{1}{\sigma_{i,t}^2}\}) \  U_t \tilde{W}_{BB,t}^{1/2}. 
\end{align*}
Letting $\displaystyle
 \Sigma_t \!\triangleq\!\!\!\!\operatorname*{diag}_{1 \leq i \leq d_{B}}\!\!(\max\{0, \frac{2}{\alpha}\!-\!\frac{1}{\sigma_{i,t}^2}\})$ completes the proof.

\bibliographystyle{ieeetr}  % or another style like ieee, elsarticle-num, etc.
\bibliography{TRO/ref_TRO}     % ref.bib file (without extension)

%%%%% Biography section
% Biography text here without a photo.

% \begin{IEEEbiography}[{\includegraphics[width=1in,height=1.25in,clip,keepaspectratio]{fig1.png}}]{IEEE Publications Technology Team}
% In this paragraph you can place your educational, professional background and research and other interests.\end{IEEEbiography}

\end{document}